\documentclass[twoside,11pt]{article}

\usepackage{jmlr2e}

\usepackage{amsmath}
\usepackage{graphicx}
\usepackage{algorithm}
\usepackage{algorithmic}

\usepackage{srra_notation}

\jmlrheading{1}{2000}{0-0}{4/12}{00/00}{Nir Ailon and Ron Begleiter
  and Esther Ezra}

\ShortHeadings{Active Learning Using SRRAs}{Ailon and Begleiter and Ezra}
\firstpageno{1}

\begin{document}

\title{Active Learning Using Smooth Relative Regret Approximations with Applications}

\author{\name Marina Meil\u{a} \email mmp@stat.washington.edu \\
       \addr Department of Statistics\\
       University of Washington\\
       Seattle, WA 98195-4322, USA
       \AND
       \name Michael I.\ Jordan \email jordan@cs.berkeley.edu \\
       \addr Division of Computer Science and Department of Statistics\\
       University of California\\
       Berkeley, CA 94720-1776, USA}
 \author{\name Nir Ailon \email nailon@cs.technion.ac.il \\
 \addr Department of Computer Science \\ Technion IIT \\ Haifa, Israel
 \AND
 \name Ron Begleiter \email ronbeg@cs.technion.ac.il \\
 \addr Department of Computer Science \\  Technion IIT \\  Haifa, Israel
 \AND 
 \name Esther Ezra \email esther@courant.nyu.edu \\
 \addr Coutrant Institute of Maehtmatical Science \\ NYU \\ New York,
 NY, USA \\
 }

\editor{?}

\maketitle

\begin{abstract}%
The disagreement coefficient of Hanneke has become a central data independent invariant in proving active learning rates.  It has been shown in various ways that a concept class with low complexity  together with a bound on  the disagreement coefficient at an optimal solution allows active learning rates that are superior to passive learning ones.

We present a different tool for pool based active learning which
follows from the existence of a certain uniform version of  low
disagreement coefficient, but is not equivalent to it. In fact, we
present two fundamental active learning problems of significant
interest for which our approach allows nontrivial active learning
bounds. However, any general purpose method relying on the
disagreement coefficient bounds only fails to guarantee any useful
bounds for these problems. The applications of interest are:
learning to rank from pairwise preferences, and clustering with side
information (a.k.a.~semi-supervised~clustering). 

The tool we use is based on the learner's ability to compute an estimator of the difference between the loss of any hypotheses and some fixed ``pivotal'' hypothesis to within an 
absolute error of at most $\eps$ times the disagreement measure ($\ell_1$ distance) between the two hypotheses.  We prove
that such an estimator implies the existence of a learning algorithm
which, at each iteration, reduces its in-class excess risk to within a constant factor.
Each iteration replaces the current pivotal hypothesis with the minimizer of the estimated loss difference function with respect to the previous
pivotal hypothesis. The label complexity essentially becomes that of computing this~estimator.

\end{abstract}


\section{Introduction}

\ron{We cite here only  works with some guarantees, there is also huge body of heuristics. Shouldn't we say explicitly that we cite theoretical works?}

Unlike in standard PAC learning, an active learner chooses which instances to learn from. In the streaming setting, they
may reject labels for instances arriving in a stream, and in the pool setting they may collect a pool of instances and then
choose a subset from which to ask labels for. Although a relatively young field compared to traditional (passive) learning, there is by now a significant body of  literature on the subject
\citep[see, e.g.,][]{QBC97_classical,Dasgupta05_cal,CastroWN05,Kaariainen06,BalcanBL06_agnost,Sugiyama06jmlr,Hanneke07,BalcanBZ07_CAL,DasguptaHM07_agnost,bach07iw,CastroN08_COLT,BalcanHW08,DasguptaH08,CavallantiCG08,Hanneke09_COLT,BeygelzimerDL09iw,BeygelzimerHLZ:nips10,Koltchinskii10active_ZZZ,Cesa-BianchiGVZ10,YangHC10,HannekeY10,El-YanivW10,HannekeS11,OrabonaC11,CavallantiCG11,YangHC11,DBLP:journals/jmlr/Wang11,jmlr:Minsker12}.
Refer to a survey by
\citet{settles.tr09} for definition of active learning.

The disagreement coefficient of \cite{Hanneke07}
has become a central data independent invariant in proving active learning rates.  It has been shown in various 
ways that a concept class with low complexity  together with a bound on  the disagreement coefficient at an optimal solution  allows active learning
rates that are superior to passive rates under certain low noise conditions \citep[see, e.g.,][]{Hanneke07,BalcanBZ07_CAL,DasguptaHM07_agnost,CastroN08_COLT,BeygelzimerHLZ:nips10}.   
The best results assuming only bounded VC dimension  $d$ and disagreement coefficient $\theta$ can roughly
be stated as follows:  If the sought (in-class) excess risk $\mu$ is the same order of magnitude as the optimal error $\nu$ or larger,
then the number of required queries is roughly $\widetilde O(\theta d \log(1/\mu))$.\footnote{The $\widetilde O$ notation suppresses polylogarithmic terms.} Otherwise, the number is roughly $\widetilde O(\theta d \nu^2/\mu^2)$.
\esther{Did we define $\tilde O(\cdot)$?}
Note that this results makes no assumption on 
the noise (except maybe for its magnitude).
Better results can be made
by assuming certain statistical properties of the noise
\citep[especially the model of][]{Mammen98smoothdiscrimination,Tsybakov04optimalaggregation}.


The idea behind the  disagreement coefficient is intuitive and simple. If a hypothesis $h$ is $r$-close to  optimal, then
the \emph{difference between their losses} (the regret of $h$) can be
computed from instances in the \emph{disagreement region} only,
defined as the set of instances on which the $r$-ball round the
optimal is not unanimous on. This means that for minimizing regret, one may restrict attention to hypotheses lying in iteratively shrinking \emph{version spaces} and to instances in the corresponding disagreement regions, which are shrinking in tandem with the version spaces if the disagreement coefficient is small.  As pointed out in
 \cite{BeygelzimerHLZ:nips10}, ignoring hypotheses outside the version space is brittle business, because a mistake in 
computation of the version space dooms the algorithm to fail.  They propose a scheme in which no version space is computed.
Instead, a certain importance weighted scheme is used.  We also use importance weighting, but in the
pool based setting and not in the streaming setting as they
do.\footnote{Note that a practitioner can pretend that any pool based
  input is a stream, though that approach would probably not take full
  advantage of the data.}

Analyzing the difference between losses of hypotheses (``relative
regrets'') is used almost in all theoretical work on active learning,
but not attached directly. In this work we argue that by carefully
constructing empirical processes uniformly estimating the relative
regret of all hypotheses with respect to a fixed ``pivotal''
hypothesis yields fast active learning rates. We call such constructions SRRA (Smooth~Relative~Regret~Approximations).  

We also show that low disagreement coefficient and VC dimension assumptions imply such efficient constructions, and give rise to yet another proof for the usefulness of the disagreement
coefficient in active learning. Our resulted algorithm that does \emph{not}
need to compute or restrict itself to shrinking version spaces. We
then present two fundamental pool based learning problems for which
direct SRRA construction yields superior active learning rates, whereas any known argument that uses  the disagreement coefficient  only, requires the practitioner to 
obtain labels for the entire pool (!) even  for moderately chosen parameters.  We conclude that the SRRA method is, up to minor factors, at least as good as the disagreement coefficient method, but can be significantly better in certain cases.

We note that another important line of design and analysis of active learning algorithms makes certain structural or bayesian 
assumptions on the noise \citep[e.g.,][]{BalcanBZ07_CAL,CastroN08_COLT,Hanneke09_COLT,Koltchinskii10active_ZZZ,YangHC10,DBLP:journals/jmlr/Wang11,YangHC11,jmlr:Minsker12}. 
We expect that one can get yet improved analysis in out framework
under these assumptions.  We leave this to future work.

The rest of the paper is laid out as follows:  In Section~\ref{notation} we present notations and basic definitions, including
an introduction to our method.  In Section~\ref{dis_implies_srra} we show that existence of low disagreement coefficient 
implies our method, in some sense.  Then, we present our two main applications, learning
to rank from pairwise preferences (LRPP) in Section~\ref{sec:LRPP} and clustering with side information in Section~\ref{clust_with_side}.
In Section~\ref{sec:additional_results} we present additional results and practical considerations, and in particular
how to use our methods with convex relaxations if the ERM problems that arise in the discussion are too difficult (computationally) to optimally solve.  We conclude in Section~\ref{future} and suggest future directions.

\section{Definitions and Notation}\label{notation}

We follow the notation of \citet{HannekeS11}:  Let $\X$ be an instance space, and
let $\Y = \{0,1\}$ be a label space. Denote by $\D$ the distribution over $\X\times \Y$,
with corresponding marginals $\D_\X$ and $\D_\Y$.  In this work we assume for convenience
that each label $Y$ is a deterministic function of $X$, so that if $X\sim \D_X$ then $(X,Y(X))$
is distributed according to $\D$.

By $\C$ we denote a concept class of functions mapping $\X$ to $\Y$.  The error rate
of a hypothesis $h\in \C$ equals
$$ 
\err_\D(h) = E_{(X,Y)\sim \D}[ h(X) \neq Y(X) ]\ .
$$
The noise rate $\opt$ of $\C$ is defined as $\opt = \inf_{h\in \C} \err_\D(h)$.
We will focus on the scenario in which $\opt$ is attained at an optimal hypothesis $h^*$, so that $\err_\D(h^*) = \nu$.
Define the distance $\dist(h_1,h_2)$ between two hypotheses $h_1,h_2\in \C$ as $\Pr_{X\sim \D_\X}[h_1(X) \neq h_2(X)]$;
observe that $\dist(\cdot,\cdot)$ is a pseudo-metric over pairs of hypotheses.
For a hypothesis $h\in \C$ and a number $r \geq 0$, the ball $\B(h, r)$ around $h$ of radius $r$ is defined as 
$\{h' \in \C:\ \dist(h, h') \leq r\}$.
For a set $V\subseteq \C$ of hypotheses, let $\dis(V)$ denote 
$$ 
\dis(V) = \{x \in \X: \exists h_1, h_2\in V \mbox{ s.t. } h_1(x) \neq h_2(x) \}\ .
$$

\subsection{The Disagreement Coefficient}
The disagreement coefficient of $h$ with respect to $\C$ under $\D_\X$ is defined as
\begin{equation}\label{eq:deftheta}
\theta_h = \sup_{r > 0} \frac {\Pr_{\D_\X} \left[ \dis\left(\B(h, r) \right) \right] }{r}\ , 
\end{equation}
where $\Pr_{\D_\X}[\W]$  for $\W\subseteq \X$ denotes the probability measure with respect to the distribution $\D_\X$.
Define the uniform disagreement coefficient $\theta$ as $\sup_{h\in \C} \theta_h$, namely
\begin{equation}\label{eq:defthetauniform}
\theta = \sup_{h\in \C}\sup_{r > 0} \frac {\Pr_{\D_\X} \left[\dis\left( \B(h, r) \right) \right]}{r}\ .
\end{equation}

\begin{remark}\label{remark:dis}
 A useful slight variation of the definitions of $\theta_h$ and $\theta$ can be obtained by replacing $\sup_{r>0}$
with $\sup_{r\geq \nu}$ in (\ref{eq:deftheta}) and (\ref{eq:defthetauniform}).  We will explicitly say when we refer
to this variation in what follows.
\end{remark}

\subsection{Smooth Relative Regret Approximations (SRRA)}
Fix $h\in \C$ (which we call the \emph{pivotal hypothesis}).  
Denote by $\reg_h: \C\mapsto \Reals$ the function defined as
$$
\reg_h(h') = \err_\D(h') - \err_\D(h)\ .
$$ 
We call $\reg_h$ the \emph{relative regret function with respect to $h$}. 
Note that for $h=h^*$ this is simply the usual regret, or (in-class) excess risk function.

\begin{definition}
  Let $f : \C\mapsto \Reals$ be any function, and $0 <\eps < 1/5$ and $0 < \mu \leq 1$. 
  We say that $f$ is an \emph{$(\eps,\mu)$-smooth relative regret approximation ($(\eps,\mu)$-SRRA)} 
  with respect to $h$ if for all $h'\in \C$,
  $$ 
  | f(h') - \reg_h(h')| \leq \eps \cdot\left(\dist(h, h') + \mu\right)\ .
  $$
  If $\mu=0$ we simply call $f$ an \emph{$\eps$-smooth relative regret approximation} with respect to $h$.
\end{definition}
\noindent
\noindent

Although the definition is general, the applications we study in
details fall under the category of pool based active learning,  in which $\X$ is a finite set of size $N$ and
$\Pr_{\D_\X}$ is the uniform measure.  Hence, taking $\mu = O(1/N)$ is tantamount  here to $\mu = 0$.   This will be useful in what follows.
The following theorem and corollary constitute the main ingredient in our work.

\begin{theorem}
  \label{thm:1}
  Let $h \in \C$ and 
  $f$ be an $(\eps,\mu)$-SRRA with respect to $h$.
  Let $h_1 = \argmin_{h'\in \C} f(h')$.  Then 
  $$ 
  \err_\D(h_1) = \left( 1+O(\eps) \right)\opt  + O\left(\eps\cdot \errd( h) \right)+ O(\eps \mu)\ .
  $$
\end{theorem}
\begin{proof}
Let $h^* \defeq \argmin_{h' \in \C} \reg_h(h')$.
  Applying the definition of $(\eps,\mu)$-SRRA we have:
  \begin{align}
    \errd( h_1 )  &{\leq} \errd( h ) + f( h_1 ) + \eps \dist( h, h_1  ) + \eps \mu \nonumber \\
                       &{\leq}  \errd(h) + f(  h^* ) + \eps \dist( h, h_1 ) + \eps \mu \nonumber \\
                       &{\leq}  \errd( h ) + \nu - \errd( h ) + \eps \dist(h, h^*) + \eps \dist(h,h_1) + 2 \eps \mu \nonumber \\
                       &{\leq}  \nu  + \eps \Big(  2 \dist(h, h^*) + \dist( h_1, h^* ) \Big) + 2 \eps \mu. \label{eq:thm1:1a}
  \end{align}
The first inequality is from the definition of $(\eps, \mu)$-SRRA, the second is from the fact that $h_1$ minimizes $f(\cdot)$ by construction,
the third is again from definition of $(\eps, \mu)$-SRRA and from the definition of the relative regret function $\reg_h$, the fourth is by the triangle inequality.
  Now, clearly for any two hypotheses $g,g'\in C$ we have that $\dist(g,g') \leq \err_\D(g) + \err_\D(g')$ by the triangle inequality.
  The proof is completed by plugging $ \dist(h, h^*) \leq \errd( h ) + \nu$, 
  and $\dist( h_1, h^*) \leq \errd(h_1) + \nu$ into Equation~\ref{eq:thm1:1a}, 
  subtracting $\eps\cdot \errd(h_1)$ from both sides, and dividing by $(1-\eps)$. 
\end{proof}

\noindent
A simple inductive use of Theorem~\ref{thm:1} proves the following  corollary,  bounding the query complexity of an ERM based active learning algorithm (see Algorithm~\ref{alg:alg} for corresponding pseudocode).  Note that this algorithm never
restricts itself to a shrinking version space.
\begin{corollary}\label{c}
Let $h_0,h_1, h_2,\dots$ be a sequence of hypotheses in $\C$ such that for all $i\geq 1$, \qquad\qquad \linebreak 
$h_i = \argmin_{h'\in \C} f_{i-1}(h')$, where $f_{i-1}$ is an $(\eps,\mu)$-SRRA with respect to $h_{i-1}$.
Then for all $i\geq 0$,
$$ 
\err_\D(h_i) = \left(1+O(\eps)\right)\opt + O(\eps^i) \errd(h_0) + O(\eps\mu)\ .
$$
\end{corollary}
\begin{proof}
    Applying Theorem~\ref{thm:1} with $h_i$ and $h_{i-1}$, we have
  \begin{equation}
    \nonumber
    \err_\D(h_i) = \left( 1+O(\eps) \right)\nu + O\left(\eps\cdot \errd( h_{i-1}) \right) + O(\eps \mu)\ .
  \end{equation}
  Solving this recursion, one gets
$$
\errd( h_i ) = \sum_{j=1}^{i} \eps^{j-1}  \left(1+O(\eps) \right) \nu + O( \eps^i )
\cdot \errd(h_0) + O\left( \sum_{j=1}^{i}  \eps^{j} \right) \mu\ .
$$
The result follows easily by bounding geometric sums.
\end{proof}

\begin{algorithm}[!htb]
 \caption{An Active Learning Algorithm from SRRA's }\label{alg:alg}
 \label{cd:alg}
 \begin{algorithmic}[1]
   	\REQUIRE an initial solution $h_0 \in \C$,  estimation parameters $\epsilon \in (0, 1/5), \mu > 0$, and  number of iterations $T$
	\STATE $i \gets 0$
	\REPEAT 
	 	\STATE $h_{i+1} \gets \argmin_{h' \in \C}, f(h')$, where $f$ is an $(\eps,\mu)$-smooth relative regret approximation with respect to $h_i$ \label{mainalgstep}
		\STATE $i \gets i+1$
	\UNTIL $i$ equals $T$
	\RETURN $h_T$
 \end{algorithmic}
\end{algorithm}

We will show below problems of interest in which $(\eps,\mu$)-SRRA's with respect to a given hypothesis $h$ 
can be obtained  using queries $Y(X)$ at few randomly (and adaptively) selected points $X$ from the pool $\X$,
if the uniform disagreement coefficient $\theta$ is small. This will constitute another proof for the usefulness of
the disagreement coefficient in design and analysis of active learning algorithms. We then present two problems
for which a direct construction of an SRRA yields a significantly better query complexity than that guaranteed 
using the disagreement coefficient alone. 

\section{Constant Uniform Disagreement Coefficient Implies Efficient SRRA's}\label{dis_implies_srra}
We show that a bounded uniform disagreement coefficient implies existence of query efficient $(\eps,\mu)$-SRRA's.  
This constitutes yet another proof of the usefulness of the disagreement coefficient in design of 
active learning algorithms, via Algorithm~\ref{alg:alg}.  

\subsection{The Construction}\label{theconstruction}

Returning to our problem,
assume the uniform disagreement coefficient $\theta$ corresponding to $\C$ is finite and $\opt >0$.  Fix some failure probability $\delta$.
We consider the range space $\left(\X,\C^{*}\right)$, defined by
$$ \C^{*} =  \left (\bigcup_{h'\in \C} \left\{ \{X\in \X: h'(X) = 0\} \right\} \right )\cup \left ( \bigcup_{h'\in \C} \left\{ \{X\in \X: h'(X) = 1\} \right\} \right )\ .$$
In other words, $\C^{*}$ is the collection of all subsets $S \subseteq \X$,
whose elements $X \in S$ are mapped to the same value ($0$ or $1$) by $h'$, for some $h' \in \C$.
Assume $\left(\X,\C^{*}\right)$ has VC dimension $d$, and fix $h\in \C$.  
Let $L= \lceil \log \mu^{-1}\rceil$.   
Define $\X_0 = \dis(\B(h, \mu))$ and for $i = 1,2,\dots,L$, define $\X_i$ to be
$$ 
\X_i = \dis(\B(h,\mu 2^i)) \setminus \dis(\B(h, \mu 2^{i-1}))\ .
$$
Let $\eta_i = \Pr_{\D_\X}[\X_i]$ be the measure of $\X_i$, and $\delta$ an hyper-parameter.
For each $i\geq 0$ draw a sample $X_{i,1}, \dots, X_{i,m}$ of
$m~=~O\left(\eps^{-2}\theta \left(d \log \theta + \log\left(\delta^{-1}\log(1/\mu) \right) \right) \right)$ examples in $\X_i$,
each of which drawn independently from the distribution $\D_\X | \X_i$ (with repetitions). (By $\D_\X|\X_i$ we mean,
the distribution $\D_\X$ conditioned on $\X_i$.) 
We will now define an estimator function $f:\C\mapsto \Reals$ of $\reg_h$, as follows.   
For any $h'\in\C$ and $i=0,1,\dots, L$ let
$$ 
f_i(h') \defeq \eta_i m^{-1}\sum_{j=1}^m  \left( \one_{Y(X_{i,j})\neq h'(X_{i,j})} - 
\one_{Y(X_{i,j})\neq h(X_{i,j})} \right )\ .
$$
Our estimator is now defined as $f(h') \defeq \sum_{i=0}^L f_i(h')$.
We next show:
\begin{theorem}
  \label{mainthm}
  Let $f$, $h$, $h'$, $m$ be as above. 
  With probability at least $1-\delta$, $f$ is an $(\eps,\mu)$-SRRA with respect to $h$.
\end{theorem}
\begin{proof}
A main tool to be exploited in the proof is called \emph{relative $\eps$-approximations} due to \citet{Haussler92} and \citet{Li00improvedbounds}.
It is defined as follows. Let $h \in \X\mapsto \Reals^+$ be some function, and let $\mu_h = E_{X\sim \D_\X}[h(X)]$.
Let   $X_1, \ldots, X_m$ denote i.i.d. draws from $\D_\X$,
and let $\hat \mu_h \defeq \frac{1}{m} \cdot \sum_{i=1}^m h(X_i)$
denote the emprical average.
Let $\kappa > 0$ be an adjustable parameter.  
We are going to use the following measure of distance between $\mu_h$ and its estimator $\hat \mu_h$, to determine how far the latter is from the true expectations:
$$
d_{\kappa}(\mu_h,\hat \mu_{h}) = \frac{|\mu_h - \hat \mu_{h}|}{\mu_h + \hat \mu_{h} + \kappa} .
$$

This measure corresponds to a relative error when approximating $\mu_h$
by $\hat \mu_h$. 
Indeed, let $\eps > 0$ be our approximation ratio, and put $d_{\kappa}(\mu_h,\hat \mu_{h}) < \eps$.
This easily yields
\begin{equation}
  \label{eq:relative_error}
  |\mu_h - \hat \mu_{h}| < \frac{2\eps}{1-\eps} \cdot \mu_h + \frac{\eps}{1-\eps} \cdot \kappa .
\end{equation}
In other words, this implies that $|\mu_h - \hat \mu_{h}| < O(\eps)(\mu_h+ \kappa)$.

Let us fix a parameter $0 < \delta < 1$.  Assume $\C$ is a set of  $\{0,1\}$ valued functions on $\X$ of VC dimension $d$.
\citet{Li00improvedbounds} show that if one samples $$m \defeq c \left(\eps^{-2}\kappa^{-1}(d\log{\kappa^{-1}} + \log{\delta^{-1}}) \right)$$ examples
as above, then~(\ref{eq:relative_error}) holds uniformly for all $h\in \C$  with probability at least $1-\delta$.

We now apply this definition of \emph{relative $\eps$-approximations}, and the corresponding results within our context.
For any $h'$, we define the following four sets of instances:
  \begin{eqnarray*}
    R_{h'}^{++} &=& \{X\in\X:\ h'(X) = Y(X) = 1 \mbox{ and } h(X) = 0\} \\
    R_{h'}^{+-} &=& \{X\in\X:\ h'(X) = 1 \mbox{ and } h(X) = Y(X) = 0\} \\
    R_{h'}^{-+} &=& \{X\in \X:\ h'(X) = 0 \mbox{ and } h(X) = Y(X) =  1\} \\
    R_{h'}^{--} &=& \{X\in \X:\ h'(X) = Y(X) = 0 \mbox{ and } h(X) =   1\}\ . \\
  \end{eqnarray*}
  Observe that the set $\{X\in \X: h(X) \neq h'(X)\}$ equals to the disjoint union of 
  $R_{h'}^{++}$, $R_{h'}^{+-}$, $R_{h'}^{-+}$ and $R_{h'}^{--}$.
  For each $i=0,\dots, L$ and $b\in\{++,+-,-+,--\}$ let $R_{h',i}^{b} = R_{h'}^{b} \cap \X_i$.
  Let $\range^b_i = \{R_{h',i}^b: h' \in \C\}$.
  It is easy to verify that the VC dimension of the range spaces $\left (\X_i, \range^b_i\right)$
  is at most $d$. Each set in $\range^b_i$ is an intersection of a set in $\C^*$ with some fixed set.

  For any $R\subseteq \X_i$ let $\rho_i(R) = \Pr_{X\sim \D_\X|\X_i}[X \in R]$, 
  and $\hat \rho_i(R) = m^{-1}\sum_{j=1}^m \one_{X_{i,j}\in R}$.  
  Note that $\hat \rho_i(R)$ is an unbiased estimator of $\rho_i(R)$.
  
  By the choice of $m$, inequality~(\ref{eq:relative_error}), and the assumptions on $\theta$ and $\nu$ we have that with probability at least $1-\delta/L$, 
  for all $R \subseteq \range^{++}_i\cup \range^{+-}_i\cup \range^{-+}_i\cup \range^{--}_i$,
  
  \begin{equation}
    \label{eq:relative_approx}  
    \left|\rho_i(R) - \hat \rho_i(R) \right| = O(\eps) \cdot \left( \rho_i(R) +\theta^{-1}\right) ,
  \end{equation}
  and by the probability union bound we obtain that this uniformly holds for all $i=0,\dots ,L$
  with probability at least $1-\delta$.
  
  Now fix $h'\in \C$ and let $r=\dist(h,h')$.  Let $r(i) = \lceil \log (r/\mu)\rceil$.  
  By the definition of $\X_i$, $h(X) = h'(X)$ for all $X \in \X_{i}$ whenever $i > r(i)$.
  We can therefore decompose $\reg_h(h')$ as:

  \begin{eqnarray*}
    \reg_h(h') &=& \err_\D(h') - \err_\D(h) \\
    &=& \sum_{i=0}^L \eta_i \cdot \left ( \Pr_{X\sim \D_\X|\X_i}[Y(X) \neq h'(X)] - \Pr_{X\sim \D_\X|\X_i}[Y(X) \neq h(X)]  \right ) \\
    &=& \sum_{i=0}^{r(i)} \eta_i \cdot \left ( \Pr_{X\sim \D_\X|\X_i}[Y(X) \neq h'(X)] - \Pr_{X\sim \D_\X|\X_i}[Y(X) \neq h(X)]  \right ) \\
    &=& \sum_{i=0}^{r(i)} \eta_i \cdot \Big (-\rho_i(R_{h'}^{++}) +\rho_i(R_{h'}^{+-}) + \rho_i(R_{h'}^{-+}) - \rho_i(R_{h'}^{--})\Big)\ .
  \end{eqnarray*}
  On the other hand, we similarly have that
  \begin{eqnarray*}
    f(h') &=&  \sum_{i=0}^{r(i)} \eta_i \cdot \Big  (-\hat \rho_i(R_{h'}^{++}) + \hat \rho_i(R_{h'}^{+-}) + \hat \rho_i(R_{h'}^{-+}) - \hat \rho_i(R_{h'}^{--})\Big)\ .
\end{eqnarray*}
  Combining, we conclude using (\ref{eq:relative_approx}) that
  \begin{eqnarray}\label{almost}
    \left | \reg_h(h') - f(h') \right | \leq  O\left (  \eps \sum_{i=0}^{r(i)} \eta_i \cdot \Big (\rho_i(R_{h'}^{++})+\rho_i(R_{h'}^{+-})+\rho_i(R_{h'}^{-+})+\rho_i(R_{h'}^{--}) + 4\theta^{-1} \Big ) \right )
  \end{eqnarray}
  But now notice that 
  $ \sum_{i=0}^{r(i)} \eta_i \cdot \Big (\rho_i(R_{h'}^{++})+\rho_i(R_{h'}^{+-})+\rho_i(R_{h'}^{-+})+\rho_i(R_{h'}^{--})\Big)$ 
  equals $r$, since it corresponds to those elements $X \in \X$ on
  which $h$, $h'$ disagree. 
  Also note that $\sum_{i=0}^{r(i)}\eta_i$ is at most  $2\max\left\{\Pr_{\D_\X}\left[\dis\left( \B(h, r) \right) \right], \Pr_{\D_\X} \left[ \dis \left(\B(h, \mu)\right)\right]\right\}$.
  By the definition of $\theta$, this implies that the RHS of (\ref{almost}) is bounded by $\eps(r+\mu)$, 
  as required by the definition of $(\eps,\mu)$-SRRA.\footnote{The $O$-notation disappeared because we assume that the
constants are properly chosen in the definition of the sample size $m$.}
\end{proof}



\begin{corollary}\label{srracor}
  \label{maincor}
  An $(\eps,\mu)$-SRRA with respect to $h$ can be constructed, with probability at least $1-\delta$, using at most 
  \begin{equation}\label{srra_from_dis}
  m\left(1 + \left\lceil \log(1/\mu)\right\rceil\right) = 
  O\left (\theta \eps^{-2}\left(\log(1/\mu) \right) \left (d\log \theta 
  + \log(\delta^{-1}\log(1/\mu))\right)\right )\ 
  \end{equation} 
  label queries.
\end{corollary}

Combining Corollaries~\ref{c} and \ref{srracor} (Algorithm~\ref{alg:alg}), we obtain an active learning algorithm in the ERM setting, with query
complexity depending on the uniform  disagreement coefficient and the VC dimension.  
Assume $\delta$ is a constant.
If we are interested in excess risk of order at least that of the optimal error $\nu$,
then we may take $\eps$ to be, say, $1/5$ and achieve the sought bound by constructing $(1/5, \nu)$-SRRA's using  $\ O(\theta d (\log(1/\nu))(\log \theta))$,
once for each of $O(\log(1/\nu))$ iterations of Algorithm~\ref{alg:alg}. 
If we seek a solution with error $(1+\eps)\nu$, we would need to construct $(\eps, \nu)$-SRRA's using
$ O(\theta d\eps^{-2}(\log(1/\nu))(\log \theta))$ query labels, one for each of $O(\log(1/\nu))$ iterations of the algorithm.
 The total label query complexity
is $ O(\theta d(\log^2(1/\nu))(\log \theta))$, which is $O(\log(1/\nu))$ times the best known bounds using disagreement coefficient and VC dimension bounds only.

A few more comparison notes are in place.   First, note that in 
known arguments bounding query complexity using the disagreement coefficient, the disagreement coefficient $\theta_{h^*}$ with respect to the optimal hypothesis $h^*$ is used in the analysis, and not the uniform coefficient $\theta$. 
Also note that both in previously known results bounding query complexity using disagreement coefficient and VC dimension bounds
as well as in our result, the slight improvement described in Remark~\ref{remark:dis} applies.
In other words, 
all arguments remain valid if we replace the supremums in  (\ref{eq:deftheta}) and (\ref{eq:defthetauniform}) with $\sup_{r\geq \nu}$.



\section{Application \#1: Learning to Rank from Pairwise Preferences (LRPP)}\label{sec:LRPP}


``Learning to Rank'' takes various forms in theory and practice of learning, as well as in combinatorial optimization.
In all versions, the goal is to order a set $V$ based on constraints.

A large body of learning literature considers the following scenario:  For each $v\in V$ there is  a label on some discrete ordinal scale, and the goal is to learn how to order $V$ so as to respect induced pairwise preferences.  For example, a scale of $\{1,2,3,4,5\}$, as in hotel/restaurant star quality; where, if $u$ has a label of $5$  (``very good'') and $v$ has a label of $1$ (``very bad''), then any ordering that places
$v$ ahead of $u$ is penalized.  Note that even if the labels are noisy, the induced pairwise preferences here are always transitive, hence no combinatorial problem arises. Our work does not deal with this setting.

When the basic unit of information consists of preferences over pairs $u,v\in V$, then the problem becomes
combinatorially interesting.  In case all quadratically many pairwise preferences are given for free, the corresponding optimization problem is known as \emph{Minimum Feedback Arc-Set in Tournaments} (MFAST).\footnote{A maximization version exists as well.}
MFAST is NP-hard \citep{Alon06}. Recently, \citet{Kenyon-Mathieu:2007:RFE:1250790.1250806} show a PTAS for this (passive learning) problem.
Several important recent works address the challenge of approximating the minimum feedback arc-set problem~\citep{Ailon:2008:AII,DBLP:conf/soda/BravermanM08,Coppersmith:2010:OWN:1798596.1798608}. 

Here we consider a query efficient variant of the problem, in which each preference comes with a cost, and the goal is to produce a 
competitive solution while reducing the preference-query overhead. 
Other very recent work consider similar settings \citep{jamiesonN11nips,Ailon11:active}. \citet{jamiesonN11nips} consider a common scenario in which the 
alternatives can be characterized in terms of $d$ real-valued features and the ranking obeys the structure of the Euclidean distances between such embeddings. 
They present an active learning algorithm that requires, using \emph{average case analysis}, as few as $O(d \log n)$ labels in the noiseless case, and  $O(d \log^2 n)$ 
labels under a certain \emph{parametric} noise model. Our work uses worst-case analysis, and assumes an adversarial noise model. In this Section we analyze the pure combinatorial problem (not assuming any feature embeddings). In Section~\ref{sec:additional_results} we tackle the problem with linearly induced permutation over feature space embeddings.

\citet{Ailon11:active} consider the same setting as ours.
Our main result Corollary~\ref{cor:comb} is a slight improvement over the main result of \citet{Ailon11:active} in query complexity,
but it provides another significant improvement.   \citet{Ailon11:active} uses a querying method that is based on a divide and conquer strategy.  The weakness of such a strategy can be explained by considering an example in which we want to search a restricted
set of permutations (e.g., the setting of Section~\ref{sec:geometric}).  When dividing and conquering, the algorithm in \cite{Ailon11:active} is doomed to search  a cartesian product of
two permutations spaces (left and right).  There is no guarantee that there even exists a permutation in the restricted space that 
respects this division.  In our querying algorithm this limitation is lifted.

\subsection{Problem Definition}\label{sec:lrpp:def}

Let $V$ be a set of $n$ elements (alternatives).
The instance space $\X$ is taken to be the set of all distinct pairs of elements in $V$, namely $V\times V\setminus\big\{(u,u): u\in V\big\}$.  The distribution $\D_\X$ is uniform on $\X$.
The label function $Y : \X\mapsto \{0,1\}$ encodes a preference function satisfying $Y\big((u,v)\big) = 1-Y\big((v,u)\big)$ for all $u,v\in V$.\footnote{Note that we could have defined $\X$ to be unordered pairs of elelements in $V$ without making any assumption on $Y$.  We chose this definition for convenience in what follows.}
By convention, we think of $Y\big((u,v)\big)=1$ as a stipulation that $u$ is preferred over $v$. For convenience we will drop
the double-parentheses in what follows.

The class of solution functions $\C$ we consider is all 
$h : \X\rightarrow \{0,1\}$ such that it is \emph{skew-symmetric}: $h(u,v)=1-h(v,u)$, and \emph{transitive}: $h(u,z) \leq h(u,v) + h(v,z)$ for all distinct $u,v,z\in V$.
This is equivalent to the space of permutations over $V$, and we will 
use the notation $\pi, \sigma,\dots$ instead of $h, h', \dots$ in the remainder of the section.
We also use notation $u\prec_{\pi} v$ as a predicate equivalent to $\pi(u,v)=1$.
Endowing $\X$ with the uniform measure,
$\dist(\pi, \sigma)$ 
turns out to be (up to normalization) the well known~\kendall~distance:
	$\dist(\pi,\sigma) = N^{-1} \sum_{u\neq v} \one_{\pi(u,v)\neq \sigma(u,v)}$,
where $N\defeq n(n-1)$ is the number of all ordered pairs.




\subsection{The Weakness of Using Disagreement Coefficient Arguments}
Let us first see if we can get a useful active learning algorithm using disagreement coefficient arguments.
It has been shown in \citet{Ailon11:active} that the uniform disagreement coefficient of $\C$
is $\Omega(n)$. To see this simple fact, notice that if we start from some permutation $\pi$ and swap the 
positions of \emph{any} two elements $u,v\in V$, then we obtain a permutation of distance at most $O(1/n)$ 
away from $\pi$, hence the disagreement region of the ball of radius $O(1/n)$ around $\pi$ is the entire space $\X$.
It is also known that the VC dimension of $\C$ is $n-1$ \citep[see,][]{DBLP:conf/wsdm/RadinskyA11}.  
Using Corollary~\ref{srracor}, we conclude that we would need $\Omega(n^2)$ preference labels to obtain
an $(\eps,\mu)$-SRRA for any meaningful pair $(\eps,\mu)$.  
This is uninformative because the cardinality of $\X$ is $O(n^2)$.  A similar bound is obtained using any known
active learning bound using  disagreement coefficient
and VC-dimension bounds~only.

\begin{remark}
A slight improvement can be obtained using the refined definition of disagreement coefficients of Remark~\ref{remark:dis}.
Observe that the uniform disagreement coefficient, as well as the disagreement coefficient at the optimal solution $h^*$ becomes $\theta = \theta_{h^*} = O(1/\nu)$, if
$\nu\geq \frac{1}{n}$.\footnote{Due to symmetry, the uniform disagreement coefficient here equals $\theta_h$ for any $h\in \C$.} This improves the query complexity bound to $O(n\nu^{-1})$.  If
$\nu$ tends to $n^{-1}$ from above, in the limit this becomes a quadratic (in $n$) query complexity.
\end{remark}
 We next show how to construct more useful (in terms of query complexity) SRRA's for LRPP, \emph{for arbitrarily small $\nu$}.

\subsection{Better SRRA for LRPP}
\label{lrpp}
\def\k{p}
Consider the following idea for creating an $\eps$-SRRA for LRPP, with respect to some fixed $\pi\in \C$.  
We start by defining the following sample size parameter:
\begin{equation}\label{pdef}\k \defeq O\left(\eps^{-3}\log^3 n\right)\ .
\end{equation}  
For all $u\in V$ and for  all $i=0,1,\dots, \lceil \log n\rceil$, let $I_{u,i}$ 
Denote the set of elements $v$ such that $2^i\k \leq |\pi(u)-\pi(v)| < 2^{i+1}\k$ where, abusing notation, $\pi(u)$ is the position
of $u$ in $\pi$. For example, $\pi(u)$ is $1$ if $u$ beats all other elements, and $n$ if it is beaten by all others.
From this set, choose a random sequence  $R_{u,i}=(v_{u,i,1}, v_{u,i,2},\dots, v_{u,i,p})$ of  $p$ elements,  each chosen uniformly and independently from $I_{u,i}$.
\footnote{A variant of this sampling scheme is as follows: For each pair $(u,v)$, add it to $S$ with probability proportional to $\min\{1, \k/|\pi(u)-\pi(v)|\}$.  A similar scheme can be found in \citep{DBLP:journals/rsa/AilonCCL07, DBLP:journals/siamcomp/HalevyK07, Ailon11:active} but the strong properties proven here were not known.}

For distinct $u,v\in V$ and a permutation $\sigma\in \C$, let $C_{u,v}(\sigma)$ denote  the contribution of the pair $u,v$ to $\err_\D(\sigma)$,  namely: 
\begin{equation}\label{eq:Cdef}
C_{u,v}(\sigma) \defeq N^{-1} \one_{\sigma(u,v)\neq Y(u,v)}\ .
\end{equation} (Note that $C_{u,v} \equiv C_{v,u}$.)
Our estimator $f(\sigma)$ of $\reg_\pi(\sigma) = \err_\D(\sigma) - \err_\D(\pi)$ is defined as
\begin{equation}\label{eq:estimator}
f(\sigma) = \sum_{u\in V}\sum_{i=0}^{\lceil \log n \rceil} \frac {|I_{u,i}|}{p} \sum_{t=1}^p \big(C_{u,v_{u,i,t}}(\sigma) - C_{u,v_{u,i,t}}(\pi)\big)\ .\end{equation}
Clearly, $f(\sigma)$ is an unbiased estimator of $\reg_\pi(\sigma)$ for any $\sigma$.  Our goal is to prove that $f(\sigma)$ is an $\eps$-SRRA.

\begin{theorem}
  \label{thm:comb}
  With probability at least $1-n^{-3}$, the function $f$ is an $\eps$-SRRA with~respect~to~$\pi$.
\end{theorem}
\begin{proof} The main idea is to \emph{decompose} the difference $|f(\sigma) - \reg_\pi(\sigma)|$ vis-a-vis corresponding pieces
of  $\dist(\sigma, \pi)$. The first half of the proof is devoted to definition of such distance ``pieces.'' Then using counting and 
standard deviation-bound arguments we show that the decomposition is, with high probability, an $\eps$-SRRA. 

We start with a few definitions. 
Abusing notation, for any $\pi\in \C$ and $u\in V$ let $\pi(u)$ denote the position of $u$ in the unique permutation
that $\pi$ defines. For example, $\pi(u)=1$ if $u$ beats all other alternatives: $\pi(u,v)=1$ for all $v\neq u$; Similarly $\pi(u)=n$
if $u$ is beaten by all other~alternatives.
For any permutation $\sigma\in \C$, we define the corresponding \emph{profile} of $\sigma$ as the
vector:\footnote{For the sake of definition assume an arbitrary
  indexing such that $V = \big\{ u_i : i=1,\ldots,n\big\}$.}
$$\prof(\sigma) = \big( \sigma(u_1) - \pi(u_1), \sigma(u_2) - \pi(u_2), \dots, \sigma(u_n) - \pi(u_n)  \big).$$
Note that $\norm{ \prof(\sigma) }_1$ is the well-known \emph{Spearman Footrule} distance between $\sigma$ and $\pi$, which we denote by $\dsf(\sigma)$ for brevity.
For a subset $V'$ of $V$, we let
$\prof(\sigma)[V']$ denote the restriction of the vector $\prof(\sigma)$ to $V'$. 
Namely, the vector obtained by zeroing in $\prof(\sigma)$ all coordinates $v\not \in V'$.

Now fix $\sigma\in \C$ and two distinct $u,v\in V$.  Assume $u,v$ is an inversion in $\sigma$ with respect to $\pi$,
and that $|\pi(u)-\pi(v)| = b$ for some integer $b$.  Then either $|\pi(u)-\sigma(u)| \geq b/2$ or $|\pi(v)-\sigma(v)|\geq b/2$.
We will ``charge'' the inversion to $\argmax_{ z \in  \{u,v\}} \big\{|\pi(z)-\sigma(z)| \big\}$.\footnote{Breaking ties using some canonical rule, for example, charge to the greater of $u,v$ viewed as integers.} 
For any $u\in V$, let $\charge_\sigma(u)$ denote
the set of elements $v\in V$ such that $(u,v)$ is an inversion in $\sigma$ with respect to $\pi$, which is charged to $u$ based
on the above rule. 
The function $\reg_{\pi}(\sigma)$ can now be written as 
\begin{equation}\label{eq:deff}
\reg_{\pi} (\sigma) = 2\sum_{u\in V}\sum _{v\in \charge_\sigma(u)} \big(C_{u,v}(\sigma) - C_{u,v}(\pi)\big)\ ,	
\end{equation}
where we recall the definition of $C_{u,v}$ in Equation~(\ref{eq:Cdef}).
Indeed, any pair that is not inverted contributes nothing to the difference.  From now on, we shall remove the subscript $\pi$, because it is held
fixed.
Similarly, our estimator $f(\sigma)$ can be written as
$$f (\sigma) = 2\sum_{u\in U}\sum_{i=0}^{\lceil \log n \rceil} 
		\,\frac{|I_{u,i}|}{p} \, \sum_{t=1}^p\big(C_{u,v}(\sigma) - C_{u,v}(\pi)\big) \one_{v_{u,i,t} \in \charge_\sigma(u)}\ .$$
For any even integer $M$ 
let $U_{\sigma, M}$ denote the set of all elements $u\in V$ such
that $$M/2 < |\pi(u)-\sigma(u)| \leq M\ .$$  Let $U_{\sigma, \leq M}$ denote: $$ \bigcup_{M'\leq M} U_{\sigma, M'}\ .$$

\noindent
Now consider the following restrictions of  $\reg(\sigma)$ and $f(\sigma)$:
\begin{eqnarray}
\reg(\sigma, M) &\defeq& 2\sum_{u\in U_{\sigma, M}}\,\sum _{v\in \charge_\sigma(u)} \big( C_{u,v}(\sigma) - C_{u,v}(\pi) \big) \label{count} \\
f(\sigma, M) &\defeq& 2\sum_{u\in U_{\sigma, M}}\sum_{i=0}^{\lceil \log n \rceil}\sum_{t=1}^p \,\frac {|I_{u,i}|}{p} \,
	 \big ( C_{u,v}(\sigma) - C_{u,v}(\pi) \big)\one_{v_{u,i,t}\in \charge_\sigma(u)} \label{countA}
\end{eqnarray}
Clearly, $f(\sigma, M)$ is an unbiased estimator of $\reg(\sigma, M)$.  
Let $T_{\sigma, M}$ denote the set of all elements $u\in V$ such that $|\pi(u)-\sigma(u)| \leq \eps M$.  
We further split the expressions in (\ref{count})-(\ref{countA}) as follows:
\begin{align}\label{count1}
\reg(\sigma,M)\defeq A(\sigma, M) + B(\sigma, M) &\text{, and   } f(\sigma,M) \defeq \hat A(\sigma, M) + \hat B(\sigma, M),
\end{align}
where,
\begin{eqnarray}\label{count1b}
A(\sigma, M) &\defeq&  2\sum_{u\in U_{\sigma, M}}\,\sum _{v\in \charge_\sigma(u)\cap \overline{T_{\sigma, M}}} \big( C_{u,v}(\sigma) - C_{u,v}(\pi) \big) \\  
\hat A(\sigma, M) &\defeq&  2\sum_{u\in U_{\sigma, M}}\,\sum_{i=0}^{\lceil \log n \rceil}\, \frac {|I_{u,i}|}{p} \,\sum_{t=1}^p  \big ( C_{u,v}(\sigma) - C_{u,v}(\pi) \big)\one_{v_{u,i,t}\in \charge_\sigma(u)\cap \overline{T_{\sigma, M}} }
\end{eqnarray}
$\overline{(\cdot)}$ is set complement in $V$, and $B(\sigma, M),\hat B(\sigma, M)$ are analogous with $T_{\sigma, M}$ instead of $\overline{T_{\sigma, M}}$, as follows:

\begin{eqnarray}\label{count1c}
B(\sigma, M) &\defeq&  2\sum_{u\in U_{\sigma, M}}\,\sum _{v\in \charge_\sigma(u)\cap {T_{\sigma, M}}} \big( C_{u,v}(\sigma) - C_{u,v}(\pi) \big) \\  
\hat B(\sigma, M) &\defeq&  2\sum_{u\in U_{\sigma, M}}\,\sum_{i=0}^{\lceil \log n \rceil}\, \frac {|I_{u,i}|}{p} \,\sum_{t=1}^p \big ( C_{u,v}(\sigma) - C_{u,v}(\pi) \big)
\one_{v_{u,i,t}\in \charge_\sigma(u)\cap{T_{\sigma, M}} }
\end{eqnarray}

We now estimate the deviation of $\hat A(\sigma, M)$ from $A(\sigma, M)$. 
Fix $M$.  
 Notice that 
the expression $A(\sigma, M)$ is completely determined by non-zero elements of the vector $\prof(\sigma)[U_{\sigma, \leq M} \cap \overline{T_{\sigma, M}}]$.  
Let $J_{\sigma, M}$ denote the number of nonzeros in $\prof(\sigma)[U_{\sigma, M}]$. 
Each nonzero coordinate of $\prof(\sigma)[U_{\sigma, \leq M} \cap \overline{T_{\sigma, M}}]$ is bounded below by $\eps M$ in absolute value by definition. 
Let $P(d,M)$ denote the number of possibilities for the vector
$\prof(\sigma)[ \overline{ T_{\sigma, M}}]$  
for $\sigma$ running over all permutations satisfying $\dsf(\sigma) = d$.
We claim that
\begin{equation}\label{boundP} P(d,M) \leq  n^{2d/(\eps M)}\ .\end{equation}  Indeed, there
can be at most $d/(\eps M)$ nonzeros in $\prof(\sigma)[ \overline{ T_{\sigma, M}}]$, and each nonzero coordinate can  trivially take at most $n$ values.  The bound follows.

Now fix integers $d$ and $J$, and consider the subspace of  permutations $\sigma$ such that $J_{\sigma, M} = J$ and $\dsf(\sigma) = d$.
Define for each $u\in U_{\sigma, M}$, $i\in \lceil \log n\rceil$
and $t=1,\dots, p$ a random variable $X_{u,i,t}$ as follows
$$ X_{u,i,t} = \frac {|I_{u,i}|} p \big ( C_{u,v}(\sigma) - C_{u,v}(\pi) \big)\one_{v_{u,i,t}\in \charge_\sigma(u)\cap \overline{T_{\sigma, M}} }\ .$$
Clearly $\hat A(\sigma, M) = 2\sum_{u \in U_{\sigma, M}} X_{u,i,t}\ .$
For any $u\in V$,  let $i_u = \argmax_i |I_{u,i}|  \leq 4M $, and observe that, by our charging scheme,
$X_{u,i,t} = 0$ almost surely, for all $i>i_u$ and $t=1\dots p$.  
Also observe that for all $u,i,t$, $|X_{u,i,t}| \leq 2N^{-1}|I_{u,i}|/p \leq 2^{i+1}/p$ almost surely.  For a random variable $X$, we denote by $\|X\|_\infty$ the infimum over numbers $\alpha$ such that $X \leq \alpha$ almost surely.
We conclude:
$$ \sum_{u\in U_{\sigma, M}} \sum_{i=0}^{i_u}\sum_{t=1}^p \|X_{u,i,t}\|_\infty^2 \leq \sum_{u\in U_{\sigma, M}}\sum_{i=0}^{i_u} N^{-2}p 2^{2i+2}/p^2 \leq c_2p^{-1}N^{-2} J M^2 $$
for some global $c_2>0$. (We used a bound on the sum of a geometric series.)
 Using Hoeffding bound, we conclude that
with the probability that $\hat A(\sigma, M)$ deviates from its
expected value of $A(\sigma, M)$ by more than some $s>0$ is at most $\exp\{-s^2p/(2c_2JM^2N^{-2})\}$.
We conclude that the probability that $A(\sigma, M)$ deviates
from its expected value by more than $\eps d/(N\log n)$
is at most $\exp\{-c_1 \eps^2 d^2 p/(JM^2\log^2 n) \}$,
for some global $c_1>0$.  Hence, by taking $p = O(\eps^{-3}d^{-1}MJ\log^3 n)$,
by union bounding over all  $P(d,M)$ possibilities
for $\prof(\sigma)[\overline {T_{\sigma, M}}]$,
with probability at least $1-n^{-7}$ simultaneously
for all $\sigma$ satisfying $J_{\sigma, M}=J$ and $\dsf(\sigma)=d$, 
\begin{equation}\label{boundA} |A(\sigma, M) - \hat A(\sigma, M)| \leq   \eps d/(N\log n)\ .\end{equation}

But note that, trivially, $JM\leq d$, hence  our choice of $p$ in (\ref{pdef}) is satisfactory.    Finally, union bound over
the $O(n^3\log n)$ possibilities for the values of $J$ and $d$ and $M=1,2,4,..$ to conclude that (\ref{boundA}) holds for all permutations $\sigma$ simultaneously, with probability at least $1-n^{-3}$.

Consider now $\hat B(\sigma, M)$ and $B(\sigma, M)$.
We will need to further decompose these two expressions
as follows.  For $u\in U_{\sigma, M}$,we define a disjoint cover $(T^1_{u,\sigma, M}, T^2_{u, \sigma, M})$ of $\charge_\sigma(u)\cap T_{\sigma, M}$ as follows.
If $\pi(u) < \sigma(u)$, then 
$$ T^1_{u,\sigma, M} \defeq \{v\in T_{\sigma, M}: \pi(u) + \eps M < \pi(v) < \sigma(u) - \eps M\}\ .$$
Otherwise,
$$ T^1_{u,\sigma, M} \defeq \{v\in T_{\sigma, M}: \sigma(u) + \eps M < \pi(v) < \pi(u) - \eps M\}\ .$$
Note that by definition, $T^1_{u,\sigma, M} \subseteq \charge_\sigma(u)$.
The set $T^2_{u,\sigma, M}$ is thus taken to be 
$$ T^2_{u, \sigma, M} \defeq (\charge_\sigma(u)\cap T_{\sigma, M}) \setminus T^1_{u,\sigma, M}\ .$$

The expressions $B(\sigma, M), \hat B(\sigma, M)$  now decompose as $B^1(\sigma, M) + B^2(\sigma, M)$ and $\hat B^1(\sigma, M) + \hat B^2(\sigma, M)$, respectively, as follows:
\begin{eqnarray}\label{count1c}
B^1(\sigma, M) &\defeq&  2\sum_{u\in U_{\sigma, M}}\,\sum _{v\in T^1_{u,\sigma, M}} \big( C_{u,v}(\sigma) - C_{u,v}(\pi) \big) \\  
B^2(\sigma, M) &\defeq&  2\sum_{u\in U_{\sigma, M}}\,\sum _{v\in T^2_{u,\sigma, M}} \big( C_{u,v}(\sigma) - C_{u,v}(\pi) \big) \\  
\hat B^1(\sigma, M) &\defeq&  2\sum_{u\in U_{\sigma, M}}\,\sum_{i=0}^{\lceil \log n \rceil}\, \frac {|I_{u,i}|}{p} \,\sum_{t=1}^p \big ( C_{u,v}(\sigma) - C_{u,v}(\pi) \big)
\one_{v_{u,i,t}\in T^1(u,\sigma, M) } \\
\hat B^2(\sigma, M) &\defeq&  2\sum_{u\in U_{\sigma, M}}\,\sum_{i=0}^{\lceil \log n \rceil}\, \frac {|I_{u,i}|}{p} \,\sum_{t=1}^p \big ( C_{u,v}(\sigma) - C_{u,v}(\pi) \big)
\one_{v_{u,i,t}\in T^2(u,\sigma, M) } \ .
\end{eqnarray}

Now notice that $B^1(\sigma, M)$ 
can be uniquely determined  from  $\prof(\sigma)[\overline {T_{\sigma, M}}]$.  Indeed, in order to identify $T^1_{u,\sigma, M}$ for some $u\in U_{\sigma, M}$, it suffices to identify
zeros in a subset of coordinates of $\prof(\sigma)[\overline {T_{\sigma, M}}]$, where the subset depends only on $\prof(\sigma)[\{u\}]$.  Additionally, the value of $C_{u,v}(\sigma) - C_{u,v}(\pi)$ can be ``read'' from $\prof(\sigma)[\overline{T_{\sigma, M}}]$ (and, of course, $Y(u,v)$) if $v\in T^1_{u, \sigma, M}$.
  Hence, a Hoeffding bound and a union bound similar to the one used for bounding $|\hat A(\sigma,M)-A(\sigma, M)|$
can be used to bound (with high probability) $|\hat B^1(\sigma, M) - B^1(\sigma,M)|$ uniformly for all $\sigma$ and $M=1,2,4,...$, as well.

\noindent
Bounding $|\hat B^2(\sigma, M) - B^2(\sigma, M)|$
can be done using the following simple claim.
\begin{claim}
For $u\in V$ and an integer $q$, we say that the sampling is successful at $(u,q)$ if the random variable 
$$\left |\left \{(i, t): \pi(v_{u,i,t}) \in [\pi(u) + (1-\eps) q, \pi(u) + (1+\eps) q] \cup   [\pi(u) - (1+\eps) q, \pi(u) - (1-\eps) q]\right \}\right |$$
is at most twice its expected value.  We say that the
sampling is successful if it is successful at all $u\in V$ and $q\leq n$.  If the sampling is successful, then uniformly for all $\sigma$ and all $M=1,2,4,...$, 
$$ |\hat B^2(\sigma, M) - B^2(\sigma, M) | = O(\eps J_{\sigma,M}M/N ) .$$
The sampling is successful with probability at least $1-n^{-3}$ if
$p=O(\eps^{-1}\log n)$.
\end{claim}
The last assertion in the claim follows from Chernoff bounds.
Note that our bound (\ref{pdef}) on $p$ is satisfactory, in virtue of the claim.

Summing up the errors $|\hat A(\sigma, M)  - A(\sigma, M)|$, $|\hat B(\sigma, M) - B(\sigma, M)|$ over all $M$ gives us the following assertion:  With probability at least $1-n^{-2}$, uniformly for all $\sigma$,
$$ |f(\sigma) - \reg_\pi(\sigma)| \leq \eps \dsf(\sigma, \pi)/(2N)\ ,$$
where $\dsf(\sigma, \pi)$ is the Spearman Footrule distance 
between $\sigma$ and $\pi$.   By \citet{DiaconisG77},
$\dist(\sigma, \pi)$ is at most twice $\dsf(\sigma, \pi)/N$.
This concludes the proof.
\end{proof}

Note that by the choice of $p$ in Equation~(\ref{pdef}), the number of preference queries required for computing $f$ is $O(\eps^{-3}n\log^3 n)$.  
We conclude from Theorem~\ref{thm:comb}, our bound on the number of preference queries, and Algorithm~\ref{alg:alg} defined in Corollary~\ref{c}, the following:
\begin{corollary}
  \label{cor:comb}
  There exists an active learning 
  algorithm for obtaining a solution $\pi\in \C$ for LRPP with
  $\err_\D(\pi)\leq \left(1+O(\eps)  \right)\nu$
  with total query complexity of $O\left(\eps^{-3} n \log^4 n\right)$.  
  The algorithm succeeds with  probability at least $1-n^{-2}$.
\end{corollary}


Corollary~\ref{cor:comb} allows us to find a solution of cost  $(1+\eps)\nu$  with query complexity that is
slightly above linear in $n$ (for constant $\eps$), regardless of the magnitude of $\nu$.  In comparison, as we saw in Section~\ref{sec:LRPP}, known active learning results (and in particular Corollary~\ref{srracor})  that used  disagreement coefficient and VC dimension bounds only
guaranteed a query complexity of $\Omega(n\nu^{-1})$, tending to 
the pool size of $n(n-1)$ as $\nu$ becomes small.  Note that  $\nu=o(1)$ is quite realistic for this problem.
For example,  consider the following noise model.  A ground truth permutation $\pi^*$ exists,  $Y(u,v)$ is obtained as a human response to the question of preference between $u$ and $v$ with respect to $\pi^*$, and the human errs with probability proportional to $|\pi^*(u) - \pi^*(v)|^{-\rho}$. Namely, closer pairs of  item in the ground truth permutation are more prone to confuse a human labeler.  The resulting noise is $\nu = n^{-\rho}$ for some $\rho>0$.\footnote{Our work does not assume Bayesian noise, and we present this scenario for illustration purposes only.}


\section{Application \#2: Clustering with Side Information}\label{clust_with_side}
Clustering with side information is a fairly new variant of clustering first described, independently, by
\citet{Demiriz99semi-supervisedclustering}, and \citet{Ben-DorSY99}. In the machine learning community it is also widely known as \emph{semi-supervised clustering}.
There are a few alternatives for the form of feedback
providing the side-information. The most natural ones are the single item labels \citep[e.g.,][]{Demiriz99semi-supervisedclustering}, and the pairwise constraints 
\citep[e.g.,][]{Ben-DorSY99}. 

Here we consider pairwise side information: ``must''/``cannot'' link for pairs of elements $u,v\in V$.  Each such information bit comes at a cost, and must be treated frugaly.
In a combinatorial optimization theoretical setting known as \emph{correlation clustering} there is no input cost overhead, and similarity information for all (quadratically many) pairs is available.  The goal there is to optimally clean
the noise (nontransitivity).
Correlation clustering was defined in  \citep{BBC04}, and also  in \citep{Shamir:2004:CGM} under the name
\emph{cluster editing}.  Constant factor approximations are known for various minimization versions of this problems \citep{CharikarW04, Ailon:2008:AII}.  
A PTAS is known for a minimization version in which the number of clusters is fixed to be $k$ \citep{GiotisGuruswami06},
as in our setting.

In machine learning, there are two main approches for utilizing pairwise side information.
In the first approach, this information is used to fine tune or learn
a \emph{distance} function, which is then passed on to any standard clustering algorithm such as $k$-means or $k$-medians \citep[see, e.g.,][]{Klein02frominstance-level,Xing02distancemetric,Cohn03semi-supervisedclustering,ShamirT11,balcan12:active}.
The second approach, which is more related to our work, modifies  the clustering algorithms's objective
so as to incorporate the pairwise constraints \citep[see, e.g.,][]{basu05,ErikssonDSN11}. \citet{basu05} in his thesis, which also serves as a comprehensive survey, has championed this approach in conjunction with $k$-means,
and hidden Markov random field clustering algorithms.
In our work we isolate the use of information coming from pairwise clustering constraints, and separate it from the
 geometry of the problem.  In future work it would be interesting to analyze our framework in conjunction with the geometric
structure of the input.  Interestingly \cite{ErikssonDSN11} studies active learning for clustering using the geometric
input structure.  Unlike our setting, they assume either no noise or bayesian noise.

\subsection{Problem Definition}

Let $V$ be a set of points of size $n$.  Our goal now is to partition $V$ into $k$ sets (clusters), where $k$ is fixed.
In most applications, $V$ is endowed with some metric, and the practitioner uses this metric in order
to evaluate the quality of a clustering solution.
In some cases, known as \emph{semi-supervised clustering}, or \emph{clustering with side information},
additional information comes in the form of \emph{pairwise constraints}.  Such a constraint tells us
for a pair $u,v\in V$ whether they should be in the same cluster or in separate ones.  We concentrate
on using such information.

Using the notation of our framework, $\X$ denotes the set of distinct pairs of elements in $V$ (same as in Section~\ref{sec:LRPP}), and $\D_\X$ is the corresponding uniform measure.  The label $Y\big((u,v)\big)=1$ means that
$u$ and $v$ should be clustered together, and $Y\big((u,v)\big)=0$ means the opposite.   
Assume that $Y\big((u,v)\big)=Y\big((v,u)\big)$ for all $u,v$.\footnote{Equivalently, assume that $\X$ contains only unordered distinct pairs without any constraint on $Y$.  For notational purposes we preferred to define $\X$ as the set of ordered distinct pairs.}

The concept class $\C$ is the set of equivalence relations  over $V$ with at most $k$ equivalence classes.  More precisely,
Every $h\in \C$ is identified with a  disjoint cover $V_1,\dots, V_k$ of $V$ (some $V_i$'s possible empty), with $h\big((u,v)\big)=1$ if and only
if $u,v\in V_j$ for some $j$.
As usual, $Y$ may induce a non-transitive
relation (e.g.,  we could have $Y\big((u,v)\big)=Y\big((v,z)\big)=1$ and $Y\big((u,z)\big)=0$).
In what follows, we will drop the double parentheses.  Also, we will abuse notation by viewing $h$ as both an equivalence
relation and as a disjoint cover $\{V_1,\dots, V_k\}$ of $V$.
We take $\D$ to be the uniform measure on $\X$.
The error 
of $h\in \C$ is given as
$ \err_\D(h) = N^{-1}\sum_{(u,v)\in \X} \one_{h(u,v) \neq Y(u,v)}$
where, as before, $N = |\X| = n(n-1)$. We will define $\cost_{u,v}(h)$ to be the contribution  $N^{-1}\one_{h(u,v) \neq Y(u,v)}$
of $(u,v)\in \X$ to $\err_\D$.
 The distance $\dist(h,h')$ is given as
$ \dist(h,h') = N^{-1} \sum_{(u,v)\in \X} \one_{h(u,v)\neq h'(u,v)}$.


\subsection{The Ineffectiveness of Using Disagreement Coefficient Arguments}
We check again what  disagreement coefficient based arguments can contribute to this problem.
It is easy to see that the uniform disagreement coefficient of $\C$ is $\Theta(n)$.
Indeed, starting from any solution $h\in \C$ with corresponding partitioning $\{V_1,\dots, V_k\}$, consider
the partition obtained by moving an element $u\in V$ from its current part $V_j$ to some other part $V_{j'}$ for $j'\neq j$.
In other words, consider the clustering $h'\in \C$ given by $\left\{V_{j'}\cup \{u\}, V_j\setminus\{u\}\right\}\cup \bigcup_{i\not \in \{j,j'\}} \left\{V_i\right\}$.  
Observe that $\dist(h,h') = O(1/n)$. On the other hand, for any $v\in V$ and for any $u\in V$ there is a choice of $j'$ so that 
$h$ and $h'$ obtained as above would disagree on $(u,v)\in \X$.  Hence, $\Pr_{\D_\X} \left[  \dis\left(\B\left(h, O(1/n) \right) \right) \right]~=~1$.  

It is also not hard to see that the VC dimension of $\C$ is $\Theta(n)$.  Indeed, any full matching over $V$ constitutes
a set which is shattered in $\C$ (as long as $k\geq 2$, of course).  On the other hand, any set $S\subseteq \X$ of size $n$
must induce an undirected cycle on the elements of $V$.  Clearly the edges of a cycle cannot be shattered by functions in  $\C$, because if $h(u_1,u_2) = h(u_2,u_3) =\cdots = h(u_{\ell-1}, u_\ell) = 1$ for $h\in \C$, then also $h(u_1,u_\ell)=1$.

Using Corollary~\ref{srracor}, we conclude that we'd need $\Omega(n^2)$  preference labels to obtain
an $(\eps,\mu)$-SRRA for any meaningful pair $(\eps,\mu)$.  This is useless because the cardinality of $\X$ is $O(n^2)$.
As in the problem discussed in Section~\ref{sec:LRPP}, this can be improved using Remark~\ref{sec:LRPP} to $\Omega(n\nu^{-1})$,
which tends to quadratic in $n$ as $\nu$ becomes smaller.
We next show how to construct more useful  SRRA's for the problem, for arbitrarily small $\nu$.

\subsection{Better SRRA for Semi-Supervised $k$-Clustering}\label{sec:srra_for_cc}

Fix  $h\in \C$, with $h=\{V_1,\dots, V_k\}$  (we allow empty $V_i$'s).   Order the $V_i$'s so that $|V_1| \geq |V_2| \geq \cdots \geq |V_k|$.   We construct an $\eps$-SRRA with respect to $h$ as follows.
For each cluster $V_i\in h$ and for each element $u\in V_i$ we draw $k-i+1$ independent samples $S_{ui},S_{u(i+1)}, \dots, S_{uk}$ as follows.
Each sample $S_{uj}$ is a subset of $V_j$ of size $q$ (to be defined below), chosen uniformly with repetitions from $V_j$,
where
\begin{equation}\label{eq:qdef} q = c_2\max \left\{\eps^{-2}k^2, \eps^{-3} k \right\}\log n\  \end{equation} 
for some global $c_2>0$.
Note that the collection of pairs $\{(u,v) \in \X: v\in S_{ui}\mbox{ for some }i\}$ is, roughly speaking, biased in such a way
that pairs containing elements in smaller clusters (with respect to $h$) are more likely to be selected.

\noindent
We define our estimator $f$ to be, for any $h'\in \C$,
\begin{equation}\label{srraclust}
  f(h') =   \sum_{i=1}^k\frac {|V_i|}q \sum_{u\in V_i}\sum_{v\in S_{ui}} f_{u,v}(h') + 2\sum_{i=1}^k\sum_{u\in V_i}\sum_{j=i+1}^k \frac {|V_j|} q\sum_{v\in S_{uj}} f_{u,v}(h')\ ,
\end{equation}
where $ f_{u,v}(h') \defeq \cost_{u,v}(h') - \cost_{u,v}(h)$ and $\cost_{u,v}(h) \defeq N^{-1}\one_{h(u,v)\neq Y(u,v)}$.
Note that the summations over $S_{ui}$ above takes into account multiplicity of elements in the multiset $S_{ui}$.

\begin{theorem}\label{thm:kclust}
With probability at least $1-n^{-3}$  the function $f$ is an $\eps$-SRRA with respect to $h$.
\end{theorem}


Consider another $k$-clustering $h'\in \C$, with corresponding partitioning $\{V'_1, \dots, V'_k\}$ of $V$.
We can write $\dist(h,h')$ as 
$$ \dist(h,h') = \sum_{(u,v)\in \X} \dist_{u,v}(h, h')\ $$
where $\dist_{u,v}(h,h') = N^{-1}( \one_{h'(u,v)=1}\one_{h(u,v)=0} + \one_{h(u,v)=1}\one_{h'(u,v)=0})$.

Let $n_i$ denote $|V_i|$, and 
recall that $n_1 \geq n_2 \geq \cdots \geq n_k$.
In what follows, we remove the subscript in $\reg_h$ and rename it $\reg$ ($h$ is held fixed).
 The function $\reg(h')$ will now be written as:
\begin{equation}\label{aaa}
\reg(h') = \sum_{i=1}^k \sum_{u \in V_i} \left (  \sum_{v\in V_i\setminus\{u\}} \reg_{u,v}(h')  + 2\sum_{j=i+1}^k \sum_{v\in V_j} \reg_{u,v}(h')   \right )
\end{equation}
where
$$ \reg_{u,v}(h')= \cost_{u,v}(h') - \cost_{u,v}(h)\ .$$ 


Clearly for each $h'$ it holds that $f(h')$ from (\ref{srraclust}) is an unbiased estimator of $\reg(h')$.  We now analyze its error.  
For each $i,j\in [k]$ let $V_{ij}$ denote $V_i \cap V'_j$.  This captures exactly the set of elements in the $i$'th cluster
in $h$ and the $j$'th cluster in $h'$.  The distance $\dist(h, h')$ can be written as follows:
\begin{equation}\label{recdecomp}
 \dist(h,h') = N^{-1} \left (\sum_{i=1}^k \sum_{j=1}^k |V_{ij} \times (V_i\setminus V_{ij})| + 2\sum_{j=1}^k \sum_{1 \leq i_1 < i_2 \leq k}|V_{i_1 j}\times V_{i_2 j}|\right )\ .\end{equation}
We call each cartesian set product in (\ref{recdecomp}) a \emph{distance contributing rectangle}.
Note that unless a pair $(u,v)$ appears in one of the distance contributing rectangles, we have $\reg_{u,v}(h') = f_{u,v}(h')~=~0$.
Hence we can decompose $\reg(h')$ and $f(h')$ in correspondence with the distance contributing rectangles, as follows:
\def\regrec{G}
\begin{align}
\reg(h') &=  \sum_{i=1}^k\sum_{j=1}^k  \regrec_{i,j}(h') + 2\sum_{j=1}^k \sum_{1 \leq i_1 < i_2 \leq k} \regrec_{i_1,i_2, j}(h') \label{recdecompf1}\\
f(h') &= \sum_{i=1}^k\sum_{j=1}^k  F_{i,j}(h') + 2\sum_{j=1}^k \sum_{1 \leq i_1 < i_2 \leq k} F_{i_1,i_2, j}(h') \label{recdecompf2}
\end{align}

where

\begin{align}
\regrec_{i,j}(h') &= \sum_{u\in V_{ij}}\sum_{v\in V_i\setminus V_{ij}} \reg_{u,v}(h') \\
F_{i,j}(h') &= \frac{|V_i|} q \sum_{u\in V_{ij}}\sum_{v \in (V_i\setminus V_{ij})\cap S_{ui}} f_{u,v}(h')  \label{hatFijdef}\\
\regrec_{i_1, i_2, j}(h') &= \sum_{u\in V_{i_1j}} \sum_{v\in V_{i_2 j}} f_{u,v}(h') \\
F_{i_1, i_2, j}(h') &= \frac{|V_{i_2}|} q \sum_{u\in V_{i_1j}} \sum_{v\in V_{i_2 j}\cap S_{ui_2}} f_{u,v}(h') 
\label{hatFi1i2jdef} 
\end{align}
(Note that the $S_{ui}$'s are multisets, and the inner sums in (\ref{hatFijdef}) and (\ref{hatFi1i2jdef}) may count elements
multiple times.)
\begin{lemma}\label{lemrectangle1}
With probability at least $1-n^{-3}$, the following holds simultaneously for all $h'\in \C$ and all $i,j\in [k]$:
\begin{equation}\label{approx1}
|\regrec_{i,j}(h') - F_{i,j}(h')| \leq \eps N^{-1}\cdot |V_{ij}\times (V_i\setminus V_{ij})| \ .
\end{equation}
\end{lemma}
\begin{proof}
The  predicate (\ref{approx1}) (for a given $i, j$) depends only on the set $V_{i j} = V_{i}\cap V'_j$.  Given a subset $B\subseteq V_i$, we say that $h'$ $(i,j)$-realizes $B$ if $V_{ij} = B$.

Now fix $i,j$ and $B\subseteq V_i$.  Assume $h'$ $(i,j)$-realizes  $B$.  Let $\beta=|B|$ and $\gamma=|V_i|$.
Consider the random variable $F_{i,j}(h')$ (see (\ref{hatFijdef})).  Think of the sample $S_{ui}$ as a sequence $S_{ui}(1), \dots, S_{ui}(q)$, where each $S_{ui}(s)$ is chosen uniformly at random from $V_i$ for $s=1,\dots, q$.
We can now rewrite $F_{i,j}(h')$ as follows:
$$ F_{i,j}(h') = \frac \gamma q \sum_{u\in B}\sum_{s=1}^q Z\big(  S_{ui}(s)  \big)$$
where 
$$ Z(v) = \begin{cases} f_{u,v}(h') & v\in V_i\setminus V_{ij} \\ 0 & \mbox{otherwise} \end{cases}\ .$$
For all $s=1,\dots q$ the random variable $Z\big(S_{ui}(s) \big)$ is bounded by $2N^{-1}$ almost surely, and its moments satisfy:
\begin{align}
E\left[  Z\big( S_{ui}(s) \big) \right] =& \frac 1 \gamma \sum_{v\in(V_i\setminus V_{ij})} f_{u,v}(h') \nonumber \\ 
E\left[  Z\big( S_{ui}(s) \big)^2  \right] \leq& \frac {4N^{-2}(\gamma-\beta)} \gamma\ .
\end{align}
From this we conclude using Bernstein inequality that for any $t \leq  N^{-1}\beta(\gamma-\beta)$,
\begin{equation*}
\Pr \left[  \left| F_{i,j}(h') - \regrec_{i,j}(h') \right| \geq t \right] \leq \exp\left\{-\frac{qt^2}{16\gamma \beta(\gamma-\beta)N^{-2}}\right\}
\end{equation*}
Plugging in $t = \eps N^{-1}\beta(\gamma-\beta)$, we conclude 
\begin{equation*}
\Pr    \left[  \left| F_{i,j}(h') - \regrec_{i,j}(h')   \right| \geq \eps N^{-1} \beta(\gamma-\beta)  \right] \leq \exp\left\{-\frac{q\eps^2 \beta(\gamma-\beta)}{16\gamma}\right\}
\end{equation*}
Now note that the number of possible sets $B\subseteq V_i$ of size $\beta$ is at most $n^{\min\{\beta, \gamma-\beta\}}$.
Using union bound and recalling our choice of $q$, the lemma follows.

\end{proof}

\noindent
Proving the following is more involved.
\begin{lemma}\label{lemrectangle2}
With probability at least $1-n^{-3}$, the following holds uniformly for all $h'\in \C$ and for all $i_1,i_2,j\in [k]$ with $i_1 < i_2$:

\begin{equation}\label{approx2}
|\regrec_{i_1, i_2,j}(h') - F_{i_1,i_2,j}(h')| \leq \eps N^{-1} \max \left \{|V_{i_1j} \times V_{i_2 j}|, \frac{ |V_{i_1 j} \times (V_{i_1}\setminus V_{i_1j})| }k, \frac {|V_{i_2 j}\times (V_{i_2}\setminus V_{i_2 j})|} k\right \}
\end{equation}
\end{lemma}
\begin{proof}
The  predicate (\ref{approx2}) (for a given $i_1,i_2, j$) depends only on the sets $V_{i_1 j} = V_{i_1}\cap V'_j$ and $V_{i_2 j} = V_{i_2}\cap V'_j$.  
Given subsets $B_1\subseteq C_{i_2}$ and $B_2\subseteq C_{i_2}$, we say that $h'$ $(i_1, i_2,j)$-realizes $(B_1, B_2)$ if
$V_{i_1 j} = B_1$ and $V_{i_2 j} = B_2$.

We now fix $i_1<i_2, j$ 
and $B_1\subseteq V_{i_1}$, $B_2\subseteq V_{i_2}$.  Assume $h'$ $(i_1, i_2, j)$-realizes $(B_1, B_2)$.  For brevity, denote $\beta_\iota=|B_\iota|$ and $\gamma_\iota = |V_{i_\iota}|$ for $\iota=1,2$.
Using Bernstein inequality as in Lemma~\ref{lemrectangle1}, we
conclude that  \ron{I think something is wrong here. Either $t \leq N^{-1} \cdot \beta_1 \beta_2$, and then second Bernstein needs to be with $N^{-2}$, -OR- the first Bernstein should be with $N^{-1}$. }

\begin{equation}\label{bern1}\Pr[|\regrec_{i_1, i_2,j}(h') -  F_{i_1,i_2,j}(h')| > t] \leq \exp\left \{-\frac {c_3 t^2  q}{\beta_1\beta_2\gamma_2N^{-2}}\right \}\ .\end{equation}
for any $t$ in the range $\left [0, N^{-1}\beta_1\beta_2\right ]$, for some global $c_3>0$.
For $t$ in the range $(N^{-1}\beta_1\beta_2, \infty)$,
\begin{equation}\label{bern2}\Pr[|\regrec_{i_1, i_2,j}(h') - F_{i_1,i_2,j}(h')| > t] \leq \exp\left \{-\frac {c_4 t  q}{ \gamma_2N^{-1}}\right \}\ .\end{equation}
for some global $c_4$.
 We consider the following three cases.
\begin{enumerate}
\item $\beta_1\beta_2 \geq \max \{\beta_1(\gamma_1-\beta_1)/k, \beta_2(\gamma_2-\beta_2)/k\}$.  Hence,
$\beta_1 \geq (\gamma_2-\beta_2)/k, \beta_2\geq (\gamma_1-\beta_1)/k$.  In this case, we can  plug in (\ref{bern1}) to get
\begin{align}
\Pr[|\regrec_{i_1, i_2,j}(h') - F_{i_1,i_2,j}(h')| > \eps N^{-1}\beta_1\beta2] &\leq \exp\left \{-\frac {c_3 \eps^2 \beta_1 \beta_2  q}{\gamma_2}\right \}  \ . \label{case1}
\end{align}
Consider two subcases. (i) If $\beta_2 \geq \gamma_2/2$ then the RHS of (\ref{case1}) is at most $\exp\left \{-\frac {c_3 \eps^2 \beta_1  q}{2}\right \}$.  The number of possible subsets $B_1,B_2$ of sizes $\beta_1, \beta_2$ respectively is clearly at most $n^{\beta_1+(\gamma_2-\beta_2)} \leq n^{\beta_1 + k\beta_1}$. Therefore, as long as $q= O(\eps^{-2}k\log n)$ (which is satisfied by our choice (\ref{eq:qdef})), then with probability at least $1-n^{-6}$ this case is taken care of in the following sense: Simultaneously for all
$j,i_1<i_2$, all possible $\beta_1 < \gamma_1=|V_{i_1}|$, $\beta_2 < \gamma_2 = |V_{i_2}|$  satisfying the assumptions and for all $B_1\subseteq V_{i_1j}, B_2\subseteq V_{i_2j}$ of sizes $\beta_1, \beta_2$ respectively and for all $h'$ $(i_1, i_2, j)$-realizing $(B_1, B_2)$ we have that
$ |\regrec_{i_1, i_2,j}(h') -  F_{i_1,i_2,j}(h')| \leq \eps \beta_1 \beta_2\ .$  
(ii) If $\beta_2 < \gamma_2/2$  then by our assumption, $\beta_1 \geq
\gamma_2/2k$.  Hence the RHS of (\ref{case1}) is at most $\exp\left
  \{-\frac {c_3 \eps^2 \beta_2  q}{2k}\right \}$.   The number of sets
$B_1,B_2$ of sizes $\beta_1, \beta_2$ respectively is clearly at most \ron{$\Big[$ I think it should be $\beta_2$ instead of $\gamma_2$ in $n^{(\gamma_1-\beta_1)+\gamma_2}$; also by the assumption ($\beta_2 < \gamma_2/2$) it cannot be that
 $\gamma_2\leq \beta_2 \Big]$ } \nir{ Right! thanks!}
$n^{(\gamma_1-\beta_1)+\beta_2} \leq n^{\beta_2(1+k)}$ (by our assumptions and since clearly $\gamma_2\leq \beta_2$).  Therefore, as long as $q= O(\eps^{-2}k^2\log n)$ (satisfied by our choice), then with probability at least $1-n^{-6}$ this case is taken care of in the following sense:
Simultaneously for all  $j,i_1<i_2$, all possible $\beta_1 < \gamma_1=|V_{i_1}|$, $\beta_2 < \gamma_2 = |V_{i_2}|$  satisfying the assumptions and for all $B_1\subseteq V_{i_1j}, B_2\subseteq V_{i_2j}$ of sizes $\beta_1, \beta_2$ respectively and for all $h'$ $(i_1, i_2, j)$-realizing $(B_1, B_2)$ we have that
$ |\regrec_{i_1, i_2,j}(h') -  F_{i_1,i_2,j}(h')| \leq \eps \beta_1 \beta_2\ .$  


\item $\beta_2(\gamma_2-\beta_2)/k \geq \max\{\beta_1\beta_2, \beta_1(\gamma_1-\beta_1)/k\}$.  We consider two subcases.
\begin{enumerate}
\item $\eps \beta_2(\gamma_2-\beta_2)/k \leq \beta_1\beta_2$.  Using (\ref{bern1}), we get
\begin{equation}  \label{case2a} 
\Pr[|\regrec_{i_1, i_2,j}(h') - F_{i_1,i_2,j}(h')| > \eps N^{-1}\beta_2(\gamma_2-\beta_2)/k] \leq \exp\left \{-\frac {c_3 \eps^2 \beta_2(\gamma_2-\beta_2)^2  q} {k^2\beta_1\gamma_2}\right \} 
\end{equation}
Again consider two subcases.  (i) $\beta_2 \leq \gamma_2/2$.  In this case we conclude from (\ref{case2a}) 
\begin{equation}  \label{case2ai} 
\Pr[|\regrec_{i_1, i_2,j}(h') - F_{i_1,i_2,j}(h')| > \eps N^{-1}\beta_2(\gamma_2-\beta_2)/k] \leq \exp\left \{-\frac{c_3 \eps^2 \beta_2\gamma_2  q}{4k^2\beta_1} \right \} 
\end{equation}
Now note that by assumption 
\begin{equation}\label{abcd} \beta_1 \leq (\gamma_2-\beta_2)/k \leq \gamma_2/k \leq \gamma_1/k\ . \end{equation} Also by assumption, 
\begin{equation}\label{ghjk}
\beta_1 \leq \beta_2(\gamma_2-\beta_2)/(\gamma_1-\beta_1) \leq \beta_2\gamma_2/(\gamma_1-\beta_1)\ .
\end{equation}
  Plugging  (\ref{abcd}) in the RHS of  (\ref{ghjk}), we conclude that $\beta_1 \leq \beta_2 \gamma_2/(\gamma_1(1-1/k)) \leq 2\beta_2\gamma_2/\gamma_1 \leq 2\beta_2$ (the last step was in virtue of our assumption $\gamma_1 \leq \gamma_2$).  From here we conclude that the RHS of (\ref{case2ai}) is at most $\exp\left \{-\frac{c_3 \eps^2 2 \gamma_2  q}{4k^2} \right \} $. \ron{$\Big[$ Shouldn't the factor $2$ should be in the denominator in:   $\frac{c_3 \eps^2 2 \gamma_2  q}{4k^2}\Big]$}
The number of sets $B_1,B_2$ of sizes $\beta_1, \beta_2$ respectively is clearly at most $n^{\beta_1+\beta_2} \leq n^{2\beta_2+\beta_2} \leq n^{3\gamma_2}$.
Hence, as long as $q=O(\eps^{-2}k^2\log n)$ (satisfied by our assumption), with probability at least $1-n^{-6}$ simultaneously for all 
$j,i_1<i_2$, all possible $\beta_1 < \gamma_1=|V_{i_1}|$, $\beta_2 < \gamma_2 = |V_{i_2}|$  satisfying the assumptions and for all $B_1\subseteq V_{i_1j}, B_2\subseteq V_{i_2j}$ of sizes $\beta_1, \beta_2$ respectively and for all $h'$ $(i_1, i_2, j)$-realizing $(B_1, B_2)$ we have that
$ |\regrec_{i_1, i_2,j}(h') - F_{i_1,i_2,j}(h')| \leq \eps \beta_2(\gamma_2-\beta_2)/k\ .$  
In the second subcase (ii) $\beta_2 > \gamma_2/2$.  The RHS of (\ref{case2a}) is at most $\exp\left \{-\frac {c_3 \eps^2 (\gamma_2-\beta_2)^2  q} {2k^2\beta_1}\right \} $.  By our assumption, $(\gamma_2-\beta_2)/(k\beta_1) \geq 1$, hence this is at most $\exp\left \{-\frac {c_3 \eps^2 (\gamma_2-\beta_2)  q} {2k}\right \} $.  The number of sets $B_1,B_2$ of sizes $\beta_1, \beta_2$ respectively is clearly at most $n^{\beta_1+(\gamma_2-\beta_2)} \leq n^{(\gamma_2-\beta_2)/k + (\gamma_2-\beta_2)} \leq n^{2(\gamma_2-\beta_2)}$.  Therefore, as long as $q= O(\eps^{-2}k\log n)$ (satisfied by our assumption), then with probability at least $1-n^{-6}$, using a similar counting and union bound argument as above, this case is taken care of in the sense that: 
$ |\regrec_{i_1, i_2,j}(h') - G_{i_1,i_2,j}(h')| \leq \eps \beta_2(\gamma_2-\beta_2)/k\ .$  
\item  $\eps \beta_2(\gamma_2-\beta_2)/k > \beta_1\beta_2$.   We now use (\ref{bern2}) to conclude
\begin{equation}  \label{case2b} 
\Pr[|\regrec_{i_1, i_2,j}(h') - F_{i_1,i_2,j}(h')| > \eps N^{-1} \beta_2(\gamma_2-\beta_2)/k] \leq \exp\left \{-\frac {c_4 \eps \beta_2(\gamma_2-\beta_2)  q} {k\gamma_2}\right \} 
\end{equation}
We again consider the cases (i) $\beta_2 \leq \gamma_2/2$ and (ii) $\beta_2 \geq \gamma_2/2$ as above. In (i), we get  that  the RHS of (\ref{case2b}) is at most  $\exp\left \{-\frac {c_5 \eps \beta_2  q} {2k}\right \}$. Now notice that by our assumptions,
\begin{equation}\label{zzzza}\beta_1 \leq \eps(\gamma_2-\beta_2)/k \leq \gamma_2/2 \leq \gamma_1/2\ .\end{equation}
  Also by our assumptions, $\beta_1 \leq \beta_2(\gamma_2-\beta_2)/(\gamma_1-\beta_1)$, which by (\ref{zzzza}) is at most $2\beta_2\gamma_2/\gamma_1 \leq 2\beta_2$. Hence the number of
possibilities for $B_1, B_2$ is at most $n^{\beta_1+\beta_2} \leq n^{3\beta_2}$.  In (ii), we get that  the RHS of (\ref{case2b}) is at most  $\exp\left \{-\frac {c_4 \eps (\gamma_2-\beta_2)  q} {2k}\right \}$, and the number of possibilities for $B_1, B_2$ is at most
$n^{\beta_1 + (\gamma_2-\beta_2)}$ which is bounded by $n^{2(\gamma_2 - \beta_2)}$ by our assumptions.  For both (i) and (ii) taking $q=O(\eps^{-1}k\log n)$ ensures with probability at least $1-n^{-6}$, using a similar counting and union bounding argument as above, case (b) is taken care of in the sense that:
$ |\regrec_{i_1, i_2,j}(h') -  F_{i_1,i_2,j}(h')| \leq \eps N^{-1} \beta_2(\gamma_2-\beta_2)/k\ .$ 
\end{enumerate}


\item $\beta_1(\gamma_1-\beta_1)/k \geq \max\{\beta_1\beta_2, \beta_2(\gamma_2-\beta_2)/k\}$.  We consider two subcases.
\begin{enumerate}
\item $\eps \beta_1(\gamma_1-\beta_1)/k \leq \beta_1\beta_2$.  
 Using (\ref{bern1}), we get
\begin{equation}  \label{case3a} 
\Pr[|\regrec_{i_1, i_2,j}(h') -  F_{i_1,i_2,j}(h')| > \eps N^{-1} \beta_1(\gamma_1-\beta_1)/k] \leq \exp\left \{-\frac {c_3 \eps^2 \beta_1(\gamma_1-\beta_1)^2 q} {k^2\beta_2\gamma_2}\right \} \ .
\end{equation}
As before, consider case (i) in which $\beta_2 \leq \gamma_2/2$ and (ii) in which $\beta_2 \geq \gamma_2/2$.  For case (i), we notice that the RHS of (\ref{case3a}) is at most $\exp\left \{-\frac {c_3 \eps^2 \beta_2(\gamma_2-\beta_2)(\gamma_1-\beta_1) q} {k^2\beta_2\gamma_2}\right \}$ (we used the fact that $\beta_1(\gamma_1-\beta_1) \geq \beta_2(\gamma_2-\beta_2)$ by assumption).  This is hence at most $\exp\left \{-\frac {c_3 \eps^2 (\gamma_1-\beta_1) q} {2k^2}\right \}$.  The number of possibilities of $B_1, B_2$ of sizes $\beta_1,\beta_2$ is clearly at most $n^{(\gamma_1-\beta_1) + \beta_2} \leq n^{(\gamma_1-\beta_1) + (\gamma_1-\beta_1)/k} \leq n^{2(\gamma_1-\beta_1)}$.  From this we conclude that $q=O(\eps^{-2} k^2\log n)$ suffices for this case.
For case (ii), we bound the RHS of (\ref{case3a}) by $\exp\left \{-\frac {c_3 \eps^2 \beta_1(\gamma_1-\beta_1)^2 q} {2k^2\beta_2^2}\right \}$.
Using the assumption that $(\gamma_1-\beta_1)/\beta_2 \geq k $, the latter expression is upper bounded  by 
$\exp\left \{-\frac {c_3 \eps^2 \beta_1 q} {2}\right \}$.  Again by our assumptions, 
\begin{equation}\label{gth} \beta_1 \geq \beta_2(\gamma_2-\beta_2)/(\gamma_1-\beta_1) \geq (\eps(\gamma_1-\beta_1)/k)(\gamma_2-\beta_2)/(\gamma_1-\beta_1) = \eps(\gamma_2-\beta_2)/k\ .\end{equation}
The number of possibilities of $B_1,B_2$ of sizes $\beta_1,\beta_2$ is clearly at most $n^{\beta_1 + (\gamma_2-\beta_2)}$ which by (\ref{gth}) is bounded by $n^{\beta_1+k\beta_1/\eps} \leq n^{2k\beta_1/\eps}$.    >From this we conclude that as long as $q=O(\eps^{-3}k\log n)$ (satisfied by our choice), this case is taken care of in sense repeatedly explained above.
\item $\eps \beta_1(\gamma_1-\beta_1)/k > \beta_1\beta_2$.  
 Using (\ref{bern2}), we get
\begin{equation}  \label{case3b} 
\Pr[|\regrec_{i_1, i_2,j}(h') -  F_{i_1,i_2,j}(h')| > \eps N^{-1} \beta_1(\gamma_1-\beta_1)/k] \leq \exp\left \{-\frac {c_4 \eps \beta_1(\gamma_1-\beta_1)  q} {k\gamma_2}\right \} 
\end{equation}
We consider two sub-cases, (i) $\beta_1 \leq \gamma_1/2$  and (ii) $\beta_1 > \gamma_1/2$.  
In case (i), we have that
\begin{align}
 \frac {\beta_1(\gamma_1-\beta_1)}{\gamma_2} & = \frac 1 2 \frac {\beta_1(\gamma_1-\beta_1)}{\gamma_2} + \frac 1 2 \frac {\beta_1(\gamma_1-\beta_1)}{\gamma_2}  \nonumber \\
&\geq \frac 1 2 \frac {\beta_1\gamma_1}{2\gamma_2} + \frac 1 2 \frac {\beta_2(\gamma_2-\beta_2)}{\gamma_2} \nonumber  \\
&\geq \beta_1/4 + \min\{\beta_2, \gamma_2-\beta_2\}/2\ . \nonumber
\end{align}
(The last step used $\gamma_2 \geq \gamma_1$.)
Hence, the RHS of (\ref{case3b}) is bounded above by $$\exp\left \{-\frac {c_4 \eps  q(\beta_1/4 + \min\{\beta_2, \gamma_2-\beta_2\}/2)} {k}\right \}\ .$$ The number of possibilities of $B_1,B_2$ of sizes $\beta_1,\beta_2$ is clearly at most $n^{\beta_1 + \min\{\beta_2, \gamma_2-\beta_2\}}$, hence as long as $q=O(\eps^{-1}k\log n)$ (satisfied by our choice), this case is taken care of in the sense repeatedly explained above.  
In case (ii), we can upper bound the RHS of (\ref{case3b})  by $\exp\left \{-\frac {c_4 \eps \gamma_1(\gamma_1-\beta_1)  q} {2k\gamma_2}\right \}  \geq \exp\left \{-\frac {c_4 \eps (\gamma_1-\beta_1)  q} {2k}\right \} $.
The number of possibilities of $B_1, B_2$ of sizes $\beta_1, \beta_2$ is clearly at most $n^{(\gamma_1-\beta_1) + \beta_2}$ which, using our assumptions, is bounded above by $n^{(\gamma_1-\beta_1) + (\gamma_1-\beta_1)/k} \leq n^{2(\gamma_1-\beta_1)}$.  Hence, as long as
$q=O(\eps^{-1}k\log n)$, this case is taken care of in the sense repeatedly explained above.
\end{enumerate}
\end{enumerate}
This concludes the proof of the lemma.
\end{proof}
\noindent
As a consequence, we get the following:
\begin{lemma}\label{epssmoothapprox}
with probability at least $1-n^{-3}$, the following holds simultaneously for all $k$-clusterings $\C'$:
$$ | \reg(h') - f(h') | \leq 5\eps \dist(h', h)\ .$$
\end{lemma}
\begin{proof}
By the triangle inequality,
\begin{align}
| \reg(h') - f(h') | \leq &  \sum_{i=1}^k \sum_{j=1}^k |\regrec_{i,j}(h') - F_{i,j}(h')|  \nonumber \\
& + 2\sum_{j=1}^k \sum_{1 \leq i_1 < i_2 \leq k}  |\regrec_{i_1, i_2, j}(h') - F_{i_1, i_2, j}(h')| \nonumber\\
\end{align}
Using  (\ref{recdecompf1})-(\ref{recdecompf2}), then Lemmata~\ref{lemrectangle1}-~\ref{lemrectangle2} (assuming success of a high probability event), rearranging the sum and finally using (\ref{recdecomp}), we get:
\begin{align}
| \reg(h') - f(h') | 
\leq  & \sum_{i=1}^k \sum_{j=1}^k \eps N^{-1} |V_{ij} \times (V_i \setminus V_{ij})| \nonumber  \\
& + 2\eps N^{-1} \sum_{j=1}^k \sum_{i_1 < i_2} \left ( |V_{i_1j}\times V_{i_2 j}| + k^{-1} |V_{i_1 j}\times (V_{i_1} \setminus V_{i_1 j}) | + k^{-1} |V_{i_2 j}\times (V_{i_2} \setminus V_{i_2 j}) | \right ) \nonumber \\
\leq &  \sum_{i=1}^k \sum_{j=1}^k \eps  N^{-1} |V_{ij} \times (V_i \setminus V_{ij})|
           + 2\eps N^{-1} \sum_{j=1}^k \sum_{i_1 < i_2}  |V_{i_1j}\times V_{i_2 j}| \nonumber \\
& + 2\eps N^{-1}\sum_{j=1}^k  \sum_{i_1 < i_2}^k  k^{-1} |V_{i_1 j}\times (V_{i_1} \setminus V_{i_1 j}) | + 
              2\eps N^{-1} \sum_{j=1}^k  \sum_{i_1 < i_2}^k k^{-1} |V_{i_2 j}\times (V_{i_2} \setminus V_{i_2 j}) | \nonumber \\
\leq & \sum_{i=1}^k \sum_{j=1}^k \eps N^{-1}  |V_{ij} \times (V_i \setminus V_{ij})|
           + 2\eps N^{-1} \sum_{j=1}^k \sum_{i_1 < i_2}  |V_{i_1j}\times V_{i_2 j}| \nonumber \\
& + 2\eps  N^{-1}\sum_{j=1}^k  \sum_{i_1=1}^k k k^{-1} |V_{i_1 j}\times (V_{i_1} \setminus V_{i_1 j}) | + 
             2 \eps N^{-1} \sum_{j=1}^k  \sum_{i_2=1}^k k k^{-1} |V_{i_2 j}\times (V_{i_2} \setminus V_{i_2 j}) | \nonumber \\
& \leq 5 \eps N^{-1} \sum_{i=1}^k \sum_{j=1}^k  |V_{ij} \times (V_i \setminus V_{ij})| +  \eps N^{-1} \sum_{j=1}^k \sum_{i_1 < i_2}  |V_{i_1j}\times V_{i_2 j}| \nonumber \\
&\leq 5\eps \dist(h,h')\ , \nonumber
\end{align}
as required.
\end{proof}


Clearly the number of  label queries  required for obtaining the $\eps$-SRRA estimator 
is $O(n\max\{\eps^{-2}k^3, \eps^{-3}k^2\}\log n)$.
Combining the theorem with this bound and the iterative algorithm described in Corollary~\ref{c} 
 (Algorithm~\ref{alg:alg}), we obtain the following:

\begin{corollary}
  \label{cor:comb2}
  There exists an active learning 
  algorithm for obtaining a solution $h \in \C$ for semi-supervised
  $k$-clustering with $\err_\D( h)\leq \left( 1+O(\eps) \right)\nu$
  with total query complexity of $O\left( n \max\left\{\eps^{-2}k^3, \eps^{-3}k^2 \right\} \log^2 n \right)$.
  The algorithm succeeds with success probabiltiy at least $1-n^{-2}$.
\end{corollary}

We do not believe the $\eps^{-3}$ factor in the corollary is tight (see Section~\ref{future}).
As in the case of Corollary~\ref{cor:comb} and the ensuing discussion around LRPP, the result in Corollary~\ref{cor:comb2} significantly beats
known active learning results depending only on disagreement coefficient and VC dimension bounds, for small $\nu$.


\section{Additional Results and Practical Considerations}\label{sec:additional_results}


\noindent
We discuss two practical extensions of our results.

\subsection{LRPP over Linearly Induced Permutations in Constant Dimensional Feature Space}
\label{sec:geometric}

A special class of interest is known as LRPP over linearly induced permutations in constant dimensional feature space.
We use the same definition of $\X$ as in Section~\ref{sec:LRPP}, except that now each point $v\in V$ is associated 
with a feature vector, which we denote using bold face: $\bv{v}\in \Reals^d$.
The concept space $\C$ now consists only of permutations $\pi$ such that there exists a vector $\bv{w}_\pi \in \Reals^d$
satisfying
\begin{equation}\label{pifromw}
 \pi(u,v) = 1 \iff \langle \bv{w}, \bv{u} - \bv{v} \rangle > 0  \ .
\end{equation}

%
We are assuming familiarity with the theory of geometric arrangements, and refer the reader
to \citet{deBergCKO08} for further details.
Geometrically, each $(u,v)\in \X$ is viewed as a halfspace
$H_{u,v} \defeq \left\{\bv{x }:\  \langle \bv{x},  \bv{u} - \bv{v}\rangle > 0\right\}$,
whose (closure) supporting hyplerplane is 
$h_{u,v} \defeq \left\{\bv{x }:\  \langle \bv{x},  \bv{u} - \bv{v}\rangle = 0 \right\}$.
Let $\H$ be the collection of these ${n \choose 2}$ hyperplanes $\left\{h_{u,v}: (u,v) \in \X \right\}$.\footnote{Note that $h_{u,v} = h_{v,u}$.}
The collection $\C$ corresponds to the maximal dimensional cells in the underlying arrangement $\A(\H)$.
We thus call $\A(\H)$ from now on the \emph{permutation arrangement}, and we naturally identify full dimensional
cells with their induced permutations.  We denote by $\cell_\pi\subseteq \Reals^d$ the unique cell corresponding to a permutation 
$\pi\in \C$.

\paragraph{Bounding the VC dimension and disagreement coefficient.}
Using standard tools from combinatorial geometry, the VC dimension of $\C$ is at most $d-1$.
Roughly speaking, this property follows from the fact that in an arrangement of $m$ hyperplanes in $d$-space,
each of which meeting the origin, the overall number of cells is at
most $O(m^{d-1})$, see~\citet{deBergCKO08}. 
As for the uniform disagreement coefficient, we show below that it is bounded by $O(n)$.
Let $\pi\in \C$ be a permutation with a corresponding cell $\cell_\pi$ in $\A(\H)$.
The ball $\B(\pi, r)$  
is, geometrically, the closure of the union of all cells corresponding to 
``realizable'' permutations $\sigma$ satisfying $\dist(\sigma, \pi)~\leq~r$. 
The corresponding disagreement region $\dis(\B(\pi, r))$ corresponds to the set of ordered pairs (halfspaces) 
intersecting  $\B(\pi, r)$. We next show:

\begin{proposition}
  \label{prop:distgen}
  The  measure of $\dis\left( \B(\pi, r) \right)$  in $\D_\X$ is at most $8 r n$.
  \esther{In light of the new phrasing of the proof, is it cardinality or a probability measure?}\nir{  Right:  It's measure}
\end{proposition}
\begin{proof}
By \citet{DiaconisG77}, the Spearman Footrule distance between any two permutations 
$\pi$ and $\sigma$ is at most twice $N\dist(\pi, \sigma)$, where $N = n(n-1)$. 
Hence, if $\dist(\pi,\sigma)$
 is $r$, then any element $u$ could only swap locations with a set of elements located up to $2rN$ 
positions away to the right or to the left. This yields a total of $4rN$ `swap-candidates' for each $u$.
Thus, at most $4 r Nn$ inversions are possible. 
Note that each inversion corresponds to a hyperplane (unordered pair) that we cross, 
and thus the total number of ordered pairs is at most $8 r N n$.  The probability measure of this set is at most $8r n$,
because we assign equal probability of $N^{-1}$ for each possible pair in $\X$. The result follows.
\end{proof}

By the proposition we have that the disagreement coefficient $\theta$ is always bounded by $O(n)$,
establishing our bound.  We now invoke Corollary~\ref{srracor} with $\mu = O(1/n^2)$ (which is tantamount to 
$\mu=0$ for this problem, because $|\X|= O(n^2)$ and we are using the uniform measure), and conclude:
\begin{theorem}
  \label{thm:geom}
  An $\eps$-SRRA for LRPP in linearly induced permutations in $d$ dimensional feature space can be constructed,  
  with respect to any $\pi\in \C$, with probability at least $1 - \delta$, using at most
  $
  O\left(n d\eps^{-2} \log^2 n + n \eps^{-2}( \log n) \left( \log(\delta^{-1}\log n) \right)  \right) \text{\ label queries}.
  $
\end{theorem}

\noindent
Combining Theorem~\ref{thm:geom}, and the iterative algorithm described in Corollary~\ref{c}:
\begin{corollary}
  \label{cor:geom}
  There exists an 
  algorithm for obtaining a solution $\pi\in \C$ for LRPP in linearly induced permutations in $d$ dimensional 
  feature space with $\err_\D(\pi)\leq \left( 1+O(\eps) \right)\nu$
  with total query complexity of 
  \begin{equation}
    \label{eq:cor:geom}
    O\Big(   \eps^{-2}n d \log^3 n + n \eps^{-2}( \log^2 n) \left(\log(\delta^{-1}\log n)\right)   \Big)
  \end{equation}
  The algorithm succeeds with success probabiltiy at least $1-\delta$.
\end{corollary}

We compare this bound to that of Corollary~\ref{cor:comb}.  For the sake of comparison, assume $\delta = n^{-2}$, so that
(\ref{eq:cor:geom}) takes the simpler form of $O\big(\eps^{-2}n d \log^3 n/\log{(1/\eps)} \big)$.  This bound is better than that of 
Corollary~\ref{cor:comb} as long as the feature space dimension $d$ is  $O(\eps^{-2}\log n)$.  
For larger dimensions, Corollary~\ref{cor:comb} gives a better bound.
It would be interesting to obtain a smoother interpolation between the \emph{geometric} structure  coming from the feature space and the \emph{combinatorial} structure coming from permutations.
We refer the reader to \cite{jamiesonN11nips} for a recent result with improved query complexity under certain  Bayesian noise  assumptions.

\subsection{Convex Relaxations}

\label{sec:relaxations}
So far we focused on  theoretical ERM aspects only. Doing so, we made no assumptions about the computability of the step $h_i = \argmin_{h'\in \C} f_{h_{i-1}}(h')$ in Corollary~\ref{c} (Step~\ref{mainalgstep} in Algoroithm~\ref{alg:alg}).   Although ERM results are interesting in their own
right, we take an additional step and consider convex relaxations.


Instead of optimizing $\errd(h)$ over the set $\C$, assume
we are interested in optimizing $\tilde \errd( \tilde h)$ over $\tilde h\in \tilde \C$, where $\tilde \C$ is a convex set
of functions from $\X$ to $\Reals$. Also assume there is a mapping $\phi : \tilde \C \mapsto  \C$ which is used as a ``rounding''
procedure.  For example, in the setting of Section~\ref{sec:geometric} the set $\tilde \C$ consists of all vectors $\bv{w} \in \Reals^d$, and the rounding
method $\phi: \tilde \C\mapsto \C$ converts $\bv{w}$ to a permutation  $\pi$ satisfying (\ref{pifromw}).
When optimizing in $\tilde \C$, one conveniently works with a convex relaxation  $\rerrd: \tilde \C\rightarrow \Reals^+$
as surrogate for the discrete loss $\err_\D$, defined as follows
\begin{equation}\label{eq:relaxation-ineq}
 \rerrd( \tilde h ) = \EE_{(X,Y)\sim \D}\left[  \rLL\left( \tilde h(X),Y \right)\right]\ .
\end{equation}
where $\rLL: \Reals \times \{0,1\} \mapsto \Reals^+$ is some function convex in the first argument, and  satisfying 
\begin{equation*}
 \one_{\left(\phi(\tilde h)\right)(X) \neq Y} \leq c\rLL\left(\tilde h(X), Y \right) 
\end{equation*}
for all $\tilde h\in \C$ and $X\in \X$, where $c>0$ is some constant.  In words, this means that $\rLL$ upper bounds
the discrete loss (up to a factor of $c$).  A typical choice for the example in Section~\ref{sec:geometric} would be to define
for all $\bv{w}\in \tilde \C$ and $X=(u,v)\in \X$: $\bv{w}(X)= \langle \bv{w}, \bv{u} - \bv{v} \rangle$, and $\rLL (a,b) = \max\{1 - a(2b-1), 0\}$.  
Using this choice, optimizing over (\ref{eq:relaxation-ineq}) becomes the famous SVMRank with the hinge loss relaxation \citep{Herbrich:1202,Joachims:1099}:
\begin{align}
\text{Minimize\ } F(\bv{w}, \xi) =\,\,\,\,\,\,\,\,\,\,& \sum_{u,v} \xi_{u,v} \label{svm:utility} \\
\text{s.t.},\,\, \forall u,v, Y(u,v) = 1: &\quad (\bv{u}-\bv{v}) \cdot \bv{w}
\geq 1 - \xi_{u,v} \nonumber \\
\forall u,v: & \quad \xi_{u,v} \geq 0 \nonumber \\
                & \quad \norm{\bv{w}} \leq c. \nonumber 
\end{align}
\noindent
(Note: $c$ is a regularization parameter.)

We now have a natural extention of relative regret:  $\rreg_{\tilde h}(\tilde h') = \rerrd( \tilde h' ) - \rerrd(\tilde h )$.
By our assumptions on convexity, $\rreg_{\tilde h} : \tilde \C\mapsto \Reals^+$ can be efficiently optimized.
We now say that $f : \tilde \C\mapsto \Reals^+$ is an $(\eps, \mu)$-SRRA with respect to $\tilde h\in \tilde \C$ if for all
$\tilde h'\in \tilde \C$,
$$ \left| \reg_{\tilde h'}(\tilde h') - f(\tilde h') \right| \leq \eps\left(  \dist\big(  \phi(\tilde h), \phi(\tilde h')\big)  + \mu \right)\ .$$
If $\mu=0$ then we simply say that $f$ is an $\eps$-SRRA.
The following is an analogue to  Corollary~\ref{c}:
\begin{theorem}\label{thm:relaxation}
Let $\tilde h_0,\tilde h_1, \tilde h_2,\dots$ be a sequence of hypotheses in $\tilde \C$ such that for all $i\geq 1$, 
$\tilde h_i = \argmin_{\tilde h'\in \tilde \C} f_{i-1}(\tilde h')$, where $f_{i-1}$ is an $(\eps,\mu)$-SRRA with respect to $\tilde h_{i-1}$.
Then for all $i \geq 1$,
$$ 
\rerrd(h_i) = \left(1+O(\eps) \right){\tilde \opt} + O(\eps^i) \rerrd( h_0 ) + O(\eps\mu)\ ,
$$
where $\tilde \opt = \inf_{\tilde h \in \tilde C} \rerrd(\tilde h)$ and the $O$-notations may hide constants that depend on $c$.
\end{theorem}
\noindent
The proof is very similar to that of Corollary~\ref{c}, and we omit the details.
It turns out that the sampling techniques used for constructing an $\eps$-SRRA from Section~\ref{lrpp} can be used
for constructing an $\eps$-SRRA for the SVMRank relaxed version as well, 
as long as $\C$ is restricted to bounded vectors $\bv{w}$ and all the feature vectors $\bv{v}$ corresponding to $v\in V$ are bounded as well.  
It is easy to confirm that under such bounded-norm setting all arguments of Section~\ref{lrpp} follows through. 
The conclusion is that we can solve SVMRank,
in polynomial time,
to within an error of $(1+\eps)\tilde \opt$ using only $O\big( n\poly(\log n, \eps^{-1}) \big)$ preference queries.

\section{Conclusions and Future Work}\label{future}

In this work we showed that being able to estimate the relative regret function using carefully biased sampling methods
can yield  query efficient active learning algorithms.   We showed that such estimations can be obtained when
the only assumptions we make are bounds on the disagreement coefficient and the VC dimension.  This leads to active learning algorithms that almost match the best known using the same assumptions.  On the other hand, we presented
two problems of vast interest (mostly outside but increasingly inside the active learning community), for which a direct analysis
of the relative regret function produced better active learning strategies.  The two problems we studied are concerned
with learning relations over a ground set, where one problem dealt order relations and the other with equivalence relations (with 
bounded number of equivalence classes). In both problems our query complexity bounds had an undesirable factors of $\eps^{-3}$ which we believe should be reduced to $\eps^{-2}$ using more advanced measure concentration tools.  We leave this to future work.
It would also be interesting to identify other problems for which our approach
yields active learning algorithms with faster than previously known convergence rates.  Immediate candidates are
hierarchical clustering and metric learning.
Finally, for LRPP, we discussed a practical scenario in which the ground set is endowed with feature vectors.  We showed
how to take the underlying geometry into account in our framework.  We did not do so for  clustering with side information.
The work of \citet{ErikssonDSN11}  indicates that incorporating geometric information into our analysis is a fruitful 
direction to pursue.

Our work made no assumptions on the noise, except maybe for its magnitude.  Another promising future research direction would be to incorporate various standard noise assumptions  known to improve active learning rates \citep[especially the model of][]{Mammen98smoothdiscrimination,Tsybakov04optimalaggregation} within our setting.

\bibliography{srra_journal}

\begin{thebibliography}{60}
\providecommand{\natexlab}[1]{#1}
\providecommand{\url}[1]{\texttt{#1}}
\expandafter\ifx\csname urlstyle\endcsname\relax
  \providecommand{\doi}[1]{doi: #1}\else
  \providecommand{\doi}{doi: \begingroup \urlstyle{rm}\Url}\fi

\bibitem[Ailon(2012)]{Ailon11:active}
Nir Ailon.
\newblock An active learning algorithm for ranking from pairwise preferences
  with an almost optimal query complexity.
\newblock \emph{Journal of Machine Learning Research}, 13:\penalty0 137--164,
  2012.

\bibitem[Ailon et~al.(2007)Ailon, Chazelle, Comandur, and
  Liu]{DBLP:journals/rsa/AilonCCL07}
Nir Ailon, Bernard Chazelle, Seshadhri Comandur, and Ding Liu.
\newblock Estimating the distance to a monotone function.
\newblock \emph{Random Struct. Algorithms}, 31\penalty0 (3):\penalty0 371--383,
  2007.

\bibitem[Ailon et~al.(2008)Ailon, Charikar, and Newman]{Ailon:2008:AII}
Nir Ailon, Moses Charikar, and Alantha Newman.
\newblock Aggregating inconsistent information: {Ranking} and clustering.
\newblock \emph{Journal of the ACM}, 55\penalty0 (5):\penalty0 23:1--23:27,
  October 2008.

\bibitem[Alon(2006)]{Alon06}
Noga Alon.
\newblock Ranking tournaments.
\newblock \emph{SIAM Journal on Discrete Mathematics}, 20, 2006.

\bibitem[Bach(2007)]{bach07iw}
Francis~R. Bach.
\newblock Active learning for misspecified generalized linear models.
\newblock In B.~Sch\"{o}lkopf, J.~Platt, and T.~Hoffman, editors,
  \emph{Advances in Neural Information Processing Systems 19}, pages 65--72.
  MIT Press, Cambridge, MA, 2007.

\bibitem[Balcan et~al.(2006)Balcan, Beygelzimer, and
  Langford]{BalcanBL06_agnost}
Maria-Florina Balcan, Alina Beygelzimer, and John Langford.
\newblock Agnostic active learning.
\newblock In \emph{ICML}, pages 65--72, 2006.

\bibitem[Balcan et~al.(2007)Balcan, Broder, and Zhang]{BalcanBZ07_CAL}
Maria-Florina Balcan, Andrei~Z. Broder, and Tong Zhang.
\newblock Margin based active learning.
\newblock In \emph{COLT}, pages 35--50, 2007.

\bibitem[Balcan et~al.(2008)Balcan, Hanneke, and Wortman]{BalcanHW08}
Maria-Florina Balcan, Steve Hanneke, and Jennifer Wortman.
\newblock The true sample complexity of active learning.
\newblock In \emph{COLT}, pages 45--56, 2008.

\bibitem[Bansal et~al.(2004)Bansal, Blum, and Chawla]{BBC04}
Nikhil Bansal, Avrim Blum, and Shuchi Chawla.
\newblock Correlation clustering.
\newblock \emph{Machine Learning}, 56:\penalty0 89--113, 2004.

\bibitem[Basu(2005)]{basu05}
Sugato Basu.
\newblock \emph{Semi-supervised Clustering: {P}robabilistic Models, Algorithms
  and Experiments}.
\newblock PhD thesis, Department of Computer Sciences, University of Texas at
  Austin, 2005.

\bibitem[Ben-Dor et~al.(1999)Ben-Dor, Shamir, and Yakhini]{Ben-DorSY99}
Amir Ben-Dor, Ron Shamir, and Zohar Yakhini.
\newblock Clustering gene expression patterns.
\newblock \emph{Journal of Computational Biology}, 6\penalty0 (3/4):\penalty0
  281--297, 1999.

\bibitem[Beygelzimer et~al.(2009)Beygelzimer, Dasgupta, and
  Langford]{BeygelzimerDL09iw}
Alina Beygelzimer, Sanjoy Dasgupta, and John Langford.
\newblock Importance weighted active learning.
\newblock In \emph{ICML}, 2009.

\bibitem[Beygelzimer et~al.(2010)Beygelzimer, Hsu, Langford, and
  Zhang.]{BeygelzimerHLZ:nips10}
Alina Beygelzimer, Daniel Hsu, John Langford, and Tong Zhang.
\newblock Agnostic active learning without constraints.
\newblock In \emph{NIPS}, 2010.

\bibitem[Braverman and Mossel(2008)]{DBLP:conf/soda/BravermanM08}
Mark Braverman and Elchanan Mossel.
\newblock Noisy sorting without resampling.
\newblock In \emph{SODA}, pages 268--276, 2008.

\bibitem[Castro et~al.(2005)Castro, Willett, and Nowak]{CastroWN05}
Rui Castro, Rebecca Willett, and Robert Nowak.
\newblock Faster rates in regression via active learning.
\newblock In \emph{NIPS}, 2005.

\bibitem[Castro and Nowak(2008)]{CastroN08_COLT}
Rui~M. Castro and Robert~D. Nowak.
\newblock Minimax bounds for active learning.
\newblock \emph{IEEE Transactions on Information Theory}, 54\penalty0
  (5):\penalty0 2339--2353, 2008.

\bibitem[Cavallanti et~al.(2008)Cavallanti, Cesa-Bianchi, and
  Gentile]{CavallantiCG08}
Giovanni Cavallanti, Nicol{\`o} Cesa-Bianchi, and Claudio Gentile.
\newblock Linear classification and selective sampling under low noise
  conditions.
\newblock In \emph{NIPS}, pages 249--256, 2008.

\bibitem[Cavallanti et~al.(2011)Cavallanti, Cesa-Bianchi, and
  Gentile]{CavallantiCG11}
Giovanni Cavallanti, Nicol{\`o} Cesa-Bianchi, and Claudio Gentile.
\newblock Learning noisy linear classifiers via adaptive and selective
  sampling.
\newblock \emph{Machine Learning}, 83\penalty0 (1):\penalty0 71--102, 2011.

\bibitem[Cesa-Bianchi et~al.(2010)Cesa-Bianchi, Gentile, Vitale, and
  Zappella]{Cesa-BianchiGVZ10}
Nicol{\`o} Cesa-Bianchi, Claudio Gentile, Fabio Vitale, and Giovanni Zappella.
\newblock Active learning on trees and graphs.
\newblock In \emph{COLT}, pages 320--332, 2010.

\bibitem[Charikar and Wirth(2004)]{CharikarW04}
Moses Charikar and Anthony Wirth.
\newblock Maximizing quadratic programs: Extending grothendieck's inequality.
\newblock In \emph{FOCS}, pages 54--60. IEEE Computer Society, 2004.

\bibitem[Cohn et~al.(2000)Cohn, Caruana, and
  Mccallum]{Cohn03semi-supervisedclustering}
David Cohn, Rich Caruana, and Andrew Mccallum.
\newblock Semi-supervised clustering with user feedback.
\newblock unpublished manuscript, 2000.
\newblock URL
  \url{http://www.cs.umass.edu/~mccallum/papers/semisup-aaai2000s.ps}.

\bibitem[Coppersmith et~al.(2010)Coppersmith, Fleischer, and
  Rurda]{Coppersmith:2010:OWN:1798596.1798608}
Don Coppersmith, Lisa~K. Fleischer, and Atri Rurda.
\newblock Ordering by weighted number of wins gives a good ranking for weighted
  tournaments.
\newblock \emph{ACM Trans. Algorithms}, 6:\penalty0 55:1--55:13, July 2010.

\bibitem[Dasgupta(2005)]{Dasgupta05_cal}
Sanjoy Dasgupta.
\newblock Coarse sample complexity bounds for active learning.
\newblock In \emph{NIPS}, 2005.

\bibitem[Dasgupta and Hsu(2008)]{DasguptaH08}
Sanjoy Dasgupta and Daniel Hsu.
\newblock Hierarchical sampling for active learning.
\newblock In \emph{ICML}, pages 208--215, 2008.

\bibitem[Dasgupta et~al.(2007)Dasgupta, Hsu, and
  Monteleoni]{DasguptaHM07_agnost}
Sanjoy Dasgupta, Daniel Hsu, and Claire Monteleoni.
\newblock A general agnostic active learning algorithm.
\newblock In \emph{NIPS}, 2007.

\bibitem[de~Berg et~al.(2008)de~Berg, Cheong, van Kreveld, and
  Overmars]{deBergCKO08}
Mark de~Berg, Otfried Cheong, Marc van Kreveld, and Mark Overmars.
\newblock \emph{Computational geometry: Algorithms and applications}.
\newblock Springer-Verlag, Berlin, 3rd edition, 2008.

\bibitem[Demiriz et~al.(1999)Demiriz, Bennett, and
  Embrechts]{Demiriz99semi-supervisedclustering}
Ayhan Demiriz, Kristin Bennett, and Mark~J. Embrechts.
\newblock Semi-supervised clustering using genetic algorithms.
\newblock In \emph{In Artificial Neural Networks in Engineering (ANNIE-99},
  pages 809--814. ASME Press, 1999.

\bibitem[Diaconis and Graham(1977)]{DiaconisG77}
Persi Diaconis and R.~L. Graham.
\newblock Spearman's {F}ootrule as a measure of disarray.
\newblock \emph{Journal of the Royal Statistical Society}, 39\penalty0
  (2):\penalty0 262--268, 1977.

\bibitem[El-Yaniv and Wiener(2010)]{El-YanivW10}
Ran El-Yaniv and Yair Wiener.
\newblock On the foundations of noise-free selective classification.
\newblock \emph{Journal of Machine Learning Research}, 11:\penalty0 1605--1641,
  2010.

\bibitem[Eriksson et~al.(2011)Eriksson, Dasarathy, Singh, and
  Nowak]{ErikssonDSN11}
Brian Eriksson, Gautam Dasarathy, Aarti Singh, and Robert~D. Nowak.
\newblock Active clustering: Robust and efficient hierarchical clustering using
  adaptively selected similarities.
\newblock \emph{Journal of Machine Learning Research - Proceedings Track},
  15:\penalty0 260--268, 2011.

\bibitem[Freund et~al.(1997)Freund, Seung, Shamir, and Tishby]{QBC97_classical}
Yoav Freund, Sebastian~H. Seung, Eli Shamir, and Naftali Tishby.
\newblock Selective sampling using the query by committee algorithm.
\newblock \emph{Machine Learning}, 28:\penalty0 133--168, September 1997.

\bibitem[Giotis and Guruswami(2006)]{GiotisGuruswami06}
Ioannis Giotis and Venkatesan Guruswami.
\newblock Correlation clustering with a fixed number of clusters.
\newblock \emph{Theory of Computing}, 2\penalty0 (1):\penalty0 249--266, 2006.

\bibitem[Halevy and Kushilevitz(2007)]{DBLP:journals/siamcomp/HalevyK07}
Shirley Halevy and Eyal Kushilevitz.
\newblock Distribution-free property-testing.
\newblock \emph{SIAM J. Comput.}, 37\penalty0 (4):\penalty0 1107--1138, 2007.

\bibitem[Hanneke(2007)]{Hanneke07}
Steve Hanneke.
\newblock A bound on the label complexity of agnostic active learning.
\newblock In \emph{ICML}, 2007.

\bibitem[Hanneke(2009)]{Hanneke09_COLT}
Steve Hanneke.
\newblock Adaptive rates of convergence in active learning.
\newblock In \emph{COLT}, 2009.

\bibitem[Hanneke(2011)]{HannekeS11}
Steve Hanneke.
\newblock Rates of convergence in active learning.
\newblock \emph{Annals of Statistics}, 39\penalty0 (1):\penalty0 333--361,
  2011.

\bibitem[Hanneke and Yang(2010)]{HannekeY10}
Steve Hanneke and Liu Yang.
\newblock Negative results for active learning with convex losses.
\newblock \emph{Journal of Machine Learning Research - Proceedings Track},
  9:\penalty0 321--325, 2010.

\bibitem[Haussler(1992)]{Haussler92}
D.~Haussler.
\newblock Decision theoretic generalizations of the {PAC} model for neural net
  and other learning applications.
\newblock \emph{Information and Control}, 100\penalty0 (1):\penalty0 78--150,
  September 1992.

\bibitem[Herbrich et~al.(2000)Herbrich, Graepel, and Obermayer]{Herbrich:1202}
Ralf Herbrich, Thore Graepel, and Klaus Obermayer.
\newblock Large margin ranking boundaries for ordinal regression.
\newblock In \emph{Advances in Large Margin Classifiers}, chapter~7, pages
  115--132. The MIT Press, 2000.

\bibitem[Jamieson and Nowak(2011)]{jamiesonN11nips}
Kevin~G. Jamieson and Rob Nowak.
\newblock Active ranking using pairwise comparisons.
\newblock In \emph{NIPS 24}, pages 2240--2248, 2011.

\bibitem[Joachims(2002)]{Joachims:1099}
Thorsten Joachims.
\newblock Optimizing search engines using clickthrough data.
\newblock In \emph{KDD}, 2002.

\bibitem[K{\"a}{\"a}ri{\"a}inen(2006)]{Kaariainen06}
Matti K{\"a}{\"a}ri{\"a}inen.
\newblock Active learning in the non-realizable case.
\newblock In \emph{ALT}, pages 63--77, 2006.

\bibitem[Kenyon-Mathieu and
  Schudy(2007)]{Kenyon-Mathieu:2007:RFE:1250790.1250806}
Claire Kenyon-Mathieu and Warren Schudy.
\newblock How to rank with few errors.
\newblock In \emph{Proceedings of the thirty-ninth annual ACM symposium on
  Theory of computing}, STOC '07, pages 95--103, 2007.

\bibitem[Klein et~al.(2002)Klein, Kamvar, and
  Manning]{Klein02frominstance-level}
Dan Klein, Sepandar~D. Kamvar, and Christopher~D. Manning.
\newblock From instance-level constraints to space-level constraints: Making
  the most of prior knowledge in data clustering.
\newblock In \emph{ICML}, pages 307--314, 2002.

\bibitem[Koltchinskii(2010)]{Koltchinskii10active_ZZZ}
Vladimir Koltchinskii.
\newblock Rademacher complexities and bounding the excess risk in active
  learning.
\newblock \emph{Journal of Machine Learning Research}, 11:\penalty0 2457--2485,
  2010.

\bibitem[Li et~al.(2000)Li, Long, and Srinivasan]{Li00improvedbounds}
Yi~Li, Philip~M. Long, and Aravind Srinivasan.
\newblock Improved bounds on the sample complexity of learning.
\newblock \emph{Journal of Computer and System Sciences}, 62:\penalty0 2001,
  2000.

\bibitem[Mammen and Tsybakov(1999)]{Mammen98smoothdiscrimination}
Enno Mammen and Alexandre~B. Tsybakov.
\newblock Smooth discrimination analysis.
\newblock \emph{Annals of Statistics}, 27:\penalty0 1808--1829, 1999.

\bibitem[Minsker(2012)]{jmlr:Minsker12}
Stanislav Minsker.
\newblock Plug-in approach to active learning.
\newblock \emph{Journal of Machine Learning Research}, 13:\penalty0 67--90,
  2012.

\bibitem[Orabona and Cesa-Bianchi(2011)]{OrabonaC11}
Francesco Orabona and Nicol{\`o} Cesa-Bianchi.
\newblock Better algorithms for selective sampling.
\newblock In \emph{ICML}, pages 433--440, 2011.

\bibitem[Radinsky and Ailon(2011)]{DBLP:conf/wsdm/RadinskyA11}
Kira Radinsky and Nir Ailon.
\newblock Ranking from pairs and triplets: information quality, evaluation
  methods and query complexity.
\newblock In \emph{WSDM}, pages 105--114, 2011.

\bibitem[Settles(2009)]{settles.tr09}
Burr Settles.
\newblock Active learning literature survey.
\newblock Technical Report 1648, University of Wisconsin--Madison, 2009.

\bibitem[Shamir and Tishby(2011)]{ShamirT11}
Ohad Shamir and Naftali Tishby.
\newblock Spectral clustering on a budget.
\newblock \emph{Journal of Machine Learning Research - Proceedings Track},
  15:\penalty0 661--669, 2011.

\bibitem[Shamir et~al.(2004)Shamir, Sharan, and Tsur]{Shamir:2004:CGM}
Ron Shamir, Roded Sharan, and Dekel Tsur.
\newblock Cluster graph modification problems.
\newblock \emph{Discrete Applied Math}, 144:\penalty0 173--182, nov 2004.

\bibitem[Sugiyama(2006)]{Sugiyama06jmlr}
Masashi Sugiyama.
\newblock Active learning in approximately linear regression based on
  conditional expectation of generalization error.
\newblock \emph{Journal of Machine Learning Research}, 7:\penalty0 141--166,
  2006.

\bibitem[Tsybakov(2004)]{Tsybakov04optimalaggregation}
Alexandre~B. Tsybakov.
\newblock Optimal aggregation of classifiers in statistical learning.
\newblock \emph{Annals of Statistics}, 32:\penalty0 135--166, 2004.

\bibitem[Voevodski et~al.(2012)Voevodski, Balcan, R\"{o}glin, Teng, and
  Xia]{balcan12:active}
Konstantin Voevodski, Maria-Florina Balcan, Heiko R\"{o}glin, Shang-Hua Teng,
  and Yu~Xia.
\newblock Active clustering of biological sequences.
\newblock \emph{Journal of Machine Learning Research}, 13:\penalty0 203--225,
  2012.

\bibitem[Wang(2011)]{DBLP:journals/jmlr/Wang11}
Liwei Wang.
\newblock Smoothness, disagreement coefficient, and the label complexity of
  agnostic active learning.
\newblock \emph{Journal of Machine Learning Research}, 12:\penalty0 2269--2292,
  2011.

\bibitem[Xing et~al.(2002)Xing, Ng, Jordan, and Russell]{Xing02distancemetric}
Eric~P. Xing, Andrew~Y. Ng, Michael~I. Jordan, and Stuart Russell.
\newblock Distance metric learning, with application to clustering with
  side-information.
\newblock In \emph{Advances in Neural Information Processing Systems 15}, pages
  505--512. MIT Press, 2002.

\bibitem[Yang et~al.(2010)Yang, Hanneke, and Carbonell]{YangHC10}
Liu Yang, Steve Hanneke, and Jaime~G. Carbonell.
\newblock Bayesian active learning using arbitrary binary valued queries.
\newblock In \emph{ALT}, pages 50--58, 2010.

\bibitem[Yang et~al.(2011)Yang, Hanneke, and Carbonell]{YangHC11}
Liu Yang, Steve Hanneke, and Jaime~G. Carbonell.
\newblock The sample complexity of self-verifying bayesian active learning.
\newblock \emph{Journal of Machine Learning Research - Proceedings Track},
  15:\penalty0 816--822, 2011.

\end{thebibliography}

\end{document}